\documentclass[onefignum,onetabnum]{siamart190516}

\usepackage{braket,amsfonts}

\usepackage{array}

\usepackage[caption=false]{subfig}
\usepackage{pgfplots}

\newsiamthm{claim}{Claim}
\newsiamremark{remark}{Remark}
\newsiamremark{assumption}{Assumption}
\crefname{assumption}{Assumption}{Assumption}
\newsiamremark{fact}{Fact}
\crefname{fact}{Fact}{Fact}

\crefname{tufwrule}{Rule}{Rule}

\usepackage{algorithmic}

\usepackage{graphicx,epstopdf}

\Crefname{ALC@unique}{Line}{Lines}

\usepackage{amsopn}

\usepackage{xspace}
\usepackage{bold-extra}
\usepackage[most]{tcolorbox}
\usepackage{multirow}
\usepackage{amsmath}

\colorlet{texcscolor}{blue!50!black}
\colorlet{texemcolor}{red!70!black}
\colorlet{texpreamble}{red!70!black}
\colorlet{codebackground}{black!25!white!25}
\newcommand{\erml}{\text{ERM}_\ell}
\newcommand{\erm}{\text{ERM}}
\newcommand{\rone}{\mathit{{Rule\mathrm{-}SBD}}\sqrt{k}}
\newcommand{\rthree}{\mathit{{Rule\mathrm{-}DBD}}\sqrt{k}}
\newcommand{\rtwo}{\mathit{{Rule\mathrm{-}SBD}}\sqrt[4]{K}}
\newcommand{\rfour}{\mathit{{Rule\mathrm{-}DBD}}\sqrt[4]{K}}
\newcommand{\rfive}{\mathit{{Rule\mathrm{-}}}\emptyset}

\lstdefinestyle{siamlatex}{%
	style=tcblatex,
	texcsstyle=*\color{texcscolor},
	texcsstyle=[2]\color{texemcolor},
	keywordstyle=[2]\color{texemcolor},
	moretexcs={cref,Cref,maketitle,mathcal,text,headers,email,url},
}

\tcbset{%
	colframe=black!75!white!75,
	coltitle=white,
	colback=codebackground, 
	colbacklower=white, 
	fonttitle=\bfseries,
	arc=0pt,outer arc=0pt,
	top=1pt,bottom=1pt,left=1mm,right=1mm,middle=1mm,boxsep=1mm,
	leftrule=0.3mm,rightrule=0.3mm,toprule=0.3mm,bottomrule=0.3mm,
	listing options={style=siamlatex}
}

\newtcblisting[use counter=example]{example}[2][]{%
	title={Example~\thetcbcounter: #2},#1}

\newtcbinputlisting[use counter=example]{\examplefile}[3][]{%
	title={Example~\thetcbcounter: #2},listing file={#3},#1}

\DeclareTotalTCBox{\code}{ v O{} }
{ 
	fontupper=\ttfamily\color{black},
	nobeforeafter,
	tcbox raise base,
	colback=codebackground,colframe=white,
	top=0pt,bottom=0pt,left=0mm,right=0mm,
	leftrule=0pt,rightrule=0pt,toprule=0mm,bottomrule=0mm,
	boxsep=0.5mm,
	#2}{#1}

\patchcmd\newpage{\vfil}{}{}{}
\begin{tcbverbatimwrite}{tmp_\jobname_header.tex}
	\title{Using Taylor-Approximated Gradients to Improve the Frank-Wolfe Method for Empirical Risk Minimization}
	
	\author{Zikai Xiong\thanks{MIT Operations Research Center, 77 Massachusetts Avenue, Cambridge, MA 02139, USA (\email{zikai@mit.edu}). Research
			supported by AFOSR Grant Nos. FA9550-19-1-0240 and FA9550-22-1-0356.
		}  
		\and Robert M. Freund\thanks{MIT Sloan School of Management, 77 Massachusetts Avenue, Cambridge, MA 02139, USA (\email{rfreund@mit.edu}). Research
			supported by AFOSR Grant Nos. FA9550-19-1-0240 and FA9550-22-1-0356.}}
	
	\headers{Improving the Frank-Wolfe Method for ERM}{Zikai Xiong and  Robert M. Freund}
\end{tcbverbatimwrite}
	\title{Using Taylor-Approximated Gradients to Improve the Frank-Wolfe Method for Empirical Risk Minimization}
	
	\author{Zikai Xiong\thanks{MIT Operations Research Center, 77 Massachusetts Avenue, Cambridge, MA 02139, USA (\email{zikai@mit.edu}). Research
			supported by AFOSR Grant Nos. FA9550-19-1-0240 and FA9550-22-1-0356.
		}
		\and Robert M. Freund\thanks{MIT Sloan School of Management, 77 Massachusetts Avenue, Cambridge, MA 02139, USA (\email{rfreund@mit.edu}). Research
			supported by AFOSR Grant Nos. FA9550-19-1-0240 and FA9550-22-1-0356.}}
	
	\headers{Improving the Frank-Wolfe Method for ERM}{Zikai Xiong and  Robert M. Freund}

\ifpdf
\hypersetup{ pdftitle={Improve the Frank-Wolfe Method for Empirical Risk Minimization} }
\fi

\usepackage{adjustbox}

\newcommand{\ignore}[1]{}

\newcommand{\eps}{\varepsilon}

\newcommand{\calB}{{\cal B}}
\newcommand{\calC}{{\cal C}}

\newcommand{\calG}{{\cal G}}

\newcommand{\calO}{{\cal O}}

\newcommand{\calU}{{\cal U}}

\begin{document}
	\maketitle
	
	
	\begin{abstract}The Frank-Wolfe method has become increasingly useful in statistical and machine learning applications, due to the structure-inducing properties of the iterates, and especially in settings where linear minimization over the feasible set is more computationally efficient than projection. In the setting of Empirical Risk Minimization -- one of the fundamental optimization problems in statistical and machine learning -- the computational effectiveness of Frank-Wolfe methods typically grows linearly in the number of data observations $n$.  This is in stark contrast to the case for typical stochastic projection methods. In order to reduce this dependence on $n$, we look to second-order smoothness of typical smooth loss functions (least squares loss and logistic loss, for example) and we propose amending the Frank-Wolfe method with Taylor series-approximated gradients, including variants for both deterministic and stochastic settings. Compared with current state-of-the-art methods in the regime where the optimality tolerance $\varepsilon$ is sufficiently small, our methods are able to simultaneously reduce the dependence on large $n$ while obtaining optimal convergence rates of Frank-Wolfe methods, in both the convex and non-convex settings.  We also propose a novel adaptive step-size approach for which we have computational guarantees.  Last of all, we present computational experiments which show that our methods exhibit very significant speed-ups over existing methods on real-world datasets for both convex and non-convex binary classification problems.
	\end{abstract}
	
	
	

	\section{Introduction}\label{introsec}

	Our problem of interest is the following Empirical Risk Minimization (ERM) problem (also called the \textit{finite-sum} problem) with constraints: 
	\begin{equation}\label{pro: ERM}
		\text{ERM:} \ \ \ \ \    \text{minimize}_{x \in \calC} \quad F(x):= \tfrac{1}{n} \sum_{i=1}^{n} f_i(x) \ , 
	\end{equation}
	where $f_i(\cdot): \mathbb{R}^p \to \mathbb{R}$ is the loss function of observation $i$ for $i=1,\dots, n$, and $x$ is comprised of the (model) coefficients/parameters which are to be optimized over the feasibility set $\calC \subset \mathbb{R}^p$. Problem \cref{pro: ERM} is a fundamental problem in statistical and machine learning, as it is the underlying optimization problem on which models are ``trained'' or ``learned'', see  \cite{vapnik1991principles,hastie2009elements}. The number of observations $n$ can be significantly larger than the dimension $p$, and our interest here will be in the regime where $n \gg p$.
	We make the following assumptions for $\erm$ problems throughout this paper:
	\begin{assumption}\label{assump: ERM}
		The following conditions hold for the $\erm$ problem \cref{pro: ERM}:
		\begin{enumerate}
			\item[(i)] the feasibility set $\mathcal{C}$ is a compact convex set,
			\item[(ii)] for any $g \in \mathbb{R}^p$, the linear minimization problem $\arg\min_{x\in\calC} g^\top x$
			can be easily solved, 
			\item[(iii)] for all $i\in[n]$, the multivariate loss function $f_i(\cdot)$ is twice-differentiable and its gradient $\nabla f_i(\cdot)$ is $L$-Lipschitz continuous on $\calC$, namely, 
			\begin{equation}\label{eq: ERM ass lip 1}
				\| \nabla f_i(u) - \nabla f_i(v) \|_* \le L \|u-v\| ~~\text{ for any $u,v\in \calC$}
			\end{equation}
			where $\|\cdot\|_*$ is the dual norm of the norm on the space of variables, and 
			\item[(iv)] for all $i\in[n]$, the Hessian $\nabla^2 f_i(\cdot)$ is $\hat{L}$-Lipschitz continuous on $\calC$, namely
			\begin{equation}\label{eq: ERM ass lip 2}
				\| \nabla^2 f_i(u) - \nabla^2 f_i(v) \| \le \hat{L} \|u-v\|~~\text{ for any $u,v\in \calC$}
			\end{equation}
			where the norm of the Hessian matrices is the operator norm induced by the norm on the space of variables.
		\end{enumerate}
	\end{assumption}
	The first three conditions of \cref{assump: ERM} are canonical for Frank-Wolfe and other methods that rely on a Linear Minimization Oracle (LMO).  However, condition (iv) of \cref{assump: ERM} concerns the smoothness of the Hessian matrices of the loss functions, and does not typically appear in the literature.  In  \cref{hereiam} we discuss and provide  justification for this extra condition.  
	
	\subsection{Examples of oft-used loss functions with smooth second derivatives}\label{hereiam}
	
	Here we discuss three oft-used loss functions that have smooth Hessian matrices and therefore satisfy condition (iv) of  \cref{assump: ERM}.  
	
	\paragraph{Quadratic losses} In very many practical instances of $\erm$ the loss functions are simple quadratics of the form $f_i(x) := \frac{1}{2} (y_i - a_i^\top x)^2$ for some vector $a_i$ and value $y_i$, such as in LASSO instances \cite{tibshirani1996regression}, matrix completion with quadratic losses, and structured sparse matrix estimation with CUR factorization \cite{mairal2011convex,lu2021generalized}. For these loss functions the Lipschitz constants are ${L} = (\sup_{i}\|a_i\|_*)^2$ and $\hat{L} = 0$.

	
	\paragraph{Logistic regression for binary classification}
	In logistic regression we are given a training set $\{(w_i,y_i)\}_{i=1}^n$ with $w_i\in\mathbb{R}^p$ and $y_i \in \{-1,1\}$, and we suppose that $\sup_{i}\|w_i\|_* \le M$.  We do maximum likelihood estimation by computing the optimal solution to the following logistic regression optimization problem:
	\begin{equation}\label{pro: logistic regression}
		\min_{x\in\calC}~  F(x):=\tfrac{1}{n} \sum_{i=1}^{n} \ln(1+\exp( - y_i w_i^\top x)) \ . 
	\end{equation}
	Let $l_i(v) := \ln(1+\exp( - y_i  v))$ and $f_i(x):= l_i(w_i^\top x)$, whereby $F(x) = \frac{1}{n}\sum_{i=1}^n f_i(x) $.  For each $i \in [n]$, we have $l_i(v) = \ln(1+\exp(-y_iv))$,   $l_i'(v) = y_i\frac{-1}{1+\exp(y_i v)}$, $l_i''(v) = y_i^2 \frac{\exp(y_i v)}{(1+\exp(y_i v))^2}$, and $l_i'''(v) = y_i^3 \frac{\exp(y_iv) - \exp(2y_i v)}{(1+\exp(y_i v))^3}$. We note that $|l_i''(v)|$ obtains its maximal value of $\tfrac{1}{4}$ at $v=0$, and hence $|l_i''(v)| \le \tfrac{1}{4}$. Similarly, $|l_i'''(v)|$ obtains its maximal value of $\tfrac{1}{6\sqrt{3}}$ when $\exp(y_i v) = 2-\sqrt{3}$ and so $|l_i'''(v)| \le \tfrac{1}{6\sqrt{3}}$.  Therefore \cref{assump: ERM} holds for logistic losses with Lipschitz constants $L= \tfrac{1}{4}\cdot M^2$ and $\hat{L}  = \tfrac{1}{6\sqrt{3}} \cdot M^3$.

	\paragraph{Binary linear classification with non-convex losses \cite{mei2018landscape}}
	We are given a training set $\{(w_i,y_i)\}_{i=1}^n$ with $w_i\in\mathbb{R}^p$ and $y_i \in \{0,1\}$, and we suppose that $\sup_{i}\|w_i\|_* \le M$.  The task is to learn $\mathbb{P}[Y = 1|w = w_i] = \sigma(w_i^\top x)$ for a given univariate threshold function $\sigma: \mathbb{R}\to [0,1]$, which leads to the optimization problem:
	\begin{equation}\label{pro:  linear classification with nonconvex Losses}
		\min_{x\in\calC} ~ \tfrac{1}{n} \sum_{i=1}^{n} (y_i - \sigma(w_i^\top x))^2 \ ,
	\end{equation}
	where $\sigma: \mathbb{R}\to [0,1]$ is the given threshold function such as the sigmoid function $\sigma(v) = (1+\exp(-v))^{-1}$.  
	Let $l_i(\cdot) := (y_i - \sigma(\cdot))^2$ and $f_i(x):= l_i(w_i^\top x)$.  Then we have $F(x) = \frac{1}{n}\sum_{i=1}^n f_i(x) $, and we note in general that \cref{pro:  linear classification with nonconvex Losses} is a non-convex optimization problem.  We have
	$l_i'(v) = -2\sigma'(v)(y_i - \sigma(v))$, $l_i''(v) = 2(\sigma'(v))^2 - 2\sigma''(v)(y_i-\sigma(v))$, and $l_i'''(v) = 6\sigma'(v)\sigma''(v) - 2\sigma'''(v)(y_i-\sigma(v))$.
	Suppose that there exists a constant $L_\sigma$ such that $|\sigma'(v)|$, $|\sigma''(v)|$, and $|\sigma'''(v)| \le L_\sigma$ for any $v\in \mathbb{R}$.  Then it follows that $|l_i''(x)| \le 2 L_\sigma^2 + 2 L_\sigma$ and $|l_i'''(x)| \le 6 L_\sigma^2 + 2 L_\sigma$. Therefore \cref{assump: ERM} holds here with Lipschitz constants $L = (2 L_\sigma^2 + 2 L_\sigma)\cdot M^2$ and $\hat{L}= (6 L_\sigma^2 + 2 L_\sigma)\cdot M^3$.  
	
	In the often-encountered case when $\sigma$ is the sigmoid function, we have $\sigma'(x)= \frac{\exp( x)}{(1+\exp( x))^2}$, $\sigma''(x)=  \frac{\exp( x)  (\exp( x) - 1)}{(1+\exp( x))^3}$, and $\sigma'''(x) =  \frac{\exp( x) (1-4\exp( x) + \exp(2 x)) }{(1 + \exp( x))^4}$, and thus $|\sigma'(x)| \le 1/4$, $|\sigma''(x)| \le 1/{6\sqrt{3}}$ and $|\sigma'''(x)| \le 1/24$.
	Because $y_i\in\{0,1\}$ and $\sigma(x)\in[0,1]$, then $|y_i - \sigma(x)|\le 1$, and it holds that $|l_i''(x)|\le  \frac{1}{8}+\frac{1}{3\sqrt{3}}$ and $|l_i'''(x)| \le  \frac{1}{4\sqrt{3}} + \frac{1}{12}$.  Therefore $\nabla f_i(\cdot)$ is $(\frac{1}{8}+\frac{1}{3\sqrt{3}})M^2$-Lipschitz continuous and $\nabla^2 f_i(\cdot)$ is $ (\frac{1}{4\sqrt{3}} + \frac{1}{12})M^3$-Lipschitz continuous.

	\subsection{Loss functions with ``linear prediction''}

	In very many empirical risk minimization problems -- including the examples described in in \cref{hereiam} -- the loss functions exhibit ``linear prediction'' \cite{lu2021generalized}, namely each loss function is of the form $f_i(x):= l_i(w_i^\top x)$ for a feature vector $w_i$ of observation $i$, and $l_i(\cdot): \mathbb{R} \to \mathbb{R}$ is a univariate function; thus $f_i$ is a scalar function of the linear function $w_i^\top x$.  In this case \cref{pro: ERM} becomes
	\begin{equation}\label{pro: ERM with linear prediction}
		\erml : \ \ \    \text{minimize}_{x \in \calC} \quad F(x):= \tfrac{1}{n} \sum_{i=1}^{n}  l_i(w_i^\top x) \quad (\text{and } f_i(x):=l_i(w_i^\top x) ) \ .
	\end{equation}
	We refer to \cref{pro: ERM with linear prediction} as ``$\erml$'' for ``ERM with linear prediction'', because the $i$-th predicted value of interest is the linear function $w_i^\top x$ of the model coefficients $x$. This structure arises in very many canonical applications of statistical and machine learning problems, including support vector machines (SVMs), LASSO, logistic regression, and matrix completion, among others, see \cite{mei2018landscape,lu2021generalized} for other examples.  Let $W \in \mathbb{R}^{p \times n}$ be the matrix whose columns are comprised of the $n$ feature vectors $w_1,~ w_2,~ \dots,~w_n$ in \cref{pro: ERM with linear prediction}. 
	Let $w_i(\calC) \subset \mathbb{R}$ denote the range of $w_i^\top x$ over $x\in \calC$. For the specific case of empirical risk minimization with linear prediction, the following assumption is a more specialized version of the conditions in \cref{assump: ERM}:
	\begin{assumption}\label{assump: ERM with linear prediction}
		The following conditions hold for the linear prediction model problem $\erml$  \cref{pro: ERM with linear prediction}: 
		\begin{enumerate}
			\item[(i)] the feasibility set $\mathcal{C}$ is a compact convex set,
			\item[(ii)] for any $g \in \mathbb{R}^p$, the linear minimization problem $\arg\min_{x\in\calC} g^\top x$
			can be easily solved, 
			\item[(iii)] for all $i \in [n]$ the univariate loss function $l_i(\cdot)$ is twice-differentiable and its first-order derivative $l_i'(\cdot)$ is $L$-Lipschitz continuous on the range of $w_i^\top x$ over $x\in \calC$, namely
			\begin{equation}\label{assump: eq: ERM structured l'}
				|l_i'(\bar\theta) - l_i'(\hat\theta) | \le L |\bar\theta - \hat\theta| ~~\text{ for any $\bar\theta, \hat\theta \in w_i(\calC)$, and}
			\end{equation}
			
			\item[(iv)] for all $i \in [n]$  the second-order derivative $l_i''(\cdot)$ is $\hat L$-Lipschitz continuous on the range of $w_i^\top x$ over $x\in \calC$, namely
			\begin{equation}\label{assump: eq: ERM structured l''}
				|l_i''(\bar\theta) - l_i''(\hat\theta) | \le \hat L |\bar\theta - \hat\theta| ~~\text{ for any $\bar\theta, \hat\theta \in w_i(\calC)$ .}
			\end{equation}
		\end{enumerate}
	\end{assumption}
	We note that \cref{assump: ERM with linear prediction} holds for all three examples in \cref{hereiam}. When viewed as instances of $\erml$, the Lipschitz constants $L$ and $\hat{L}$ in \cref{assump: ERM with linear prediction} are independent of $\sup_i\|w_i\|_*$. For the example of quadratic losses, we have $L \le 1$ and $\hat{L}  = 0$. For the example of logistic regression for binary classification \cref{pro: logistic regression} we have $L \le \tfrac{1}{4}$ and $\hat{L} \le \tfrac{1}{6\sqrt{3}}$. For the example of binary linear classification with non-convex losses\cref{pro:  linear classification with nonconvex Losses}, we have $L \le 2 L_\sigma^2 + 2 L_\sigma$ and $\hat{L} \le 6 L_\sigma^2 + 2 L_\sigma$.  And in the special case when $\sigma(\cdot)$ is the sigmoid function, we have $L \le  \frac{1}{8}+\frac{1}{3\sqrt{3}}$ and $\hat{L} \le  \frac{1}{4\sqrt{3}} + \frac{1}{12}$.

	\subsection{Motivation and literature review}
	With an ever-increasing number $n$ of data observations in statistical and machine learning applications in $\erm$, there is growing importance in solving models where the number of observations $n$ is very large (compared, say, to the dimension $p$ of the features). The presence of large $n$ does not pose a particular problem (at least in theory) for a typical version of the stochastic gradient descent method (SGD), which forms a gradient estimate at each iteration by randomly sampling a mini-batch of observations and then projects back onto the feasibility set $\calC$ after taking a gradient-estimate-based step. The computational complexity of SGD is not adversely affected by large values of $n$. 
	
	However, in many applications of $\erm$ it is more attractive to use the Frank-Wolfe method or a version thereof, due to the structure-inducing properties of the iterates (sparsity, low-rank) especially when combined with in-face steps, see for example \cite{freund2017}.  The Frank-Wolfe method is also attractive in settings where a Linear Minimization Oracle (LMO) on the feasibility set $\calC$ is more computationally efficient than projection onto $\calC$,  where an LMO is a subroutine that solves and returns
	$$
	\arg\min_{x\in \calC} \langle g,x\rangle  \ 
	$$ for any given objective function vector $g \in \mathbb{R}^p$. For example, when $\calC$ is the nuclear norm ball, the LMO only requires computing the leading singular vector of a matrix, as compared to a projection oracle which requires a full SVD of a matrix, see \cite{jaggi2013revisiting} for example. As a result, the Frank-Wolfe method, which replaces gradient steps and projection by an LMO instead, has become increasingly useful in statistical and machine learning applications, for both convex and non-convex loss functions. However, when the number of observations $n$ is large, the Frank-Wolfe method has lacked ``good'' efficiency as compared to SGD, as the number of operations of Frank-Wolfe typically grows as $O(n/\eps)$. The primary motivation of the research in this paper therefore is to address the following question:
	\begin{itemize}
		\item[] \textit{Are there efficient stochastic (or deterministic) Frank-Wolfe methods}\vspace{-.04cm} 
		\item[] \textit{that eliminate or reduce the dependence on the number of observations $n$,}\vspace{-.04cm}
		\item[] \textit{in theory and in practice?}
	\end{itemize}\medskip
	
	In the last decade many researchers have studied stochastic versions of the Frank-Wolfe method. Before discussing these related works, let us introduce some useful measurements for Frank-Wolfe methods. We use the following Frank-Wolfe gap $\calG(x)$ as a  measure of non-stationarity:
	\begin{equation}\label{def: FW gap}
		\calG(x) := \max_{s \in \calC}\langle x - s, \nabla F(x) \rangle \ .
	\end{equation}
	The Frank-Wolfe gap is always nonnegative and is $0$ if and only if $x$ is a stationary point of \cref{pro: ERM with linear prediction}. 
	For problems with convex losses $F$, supposing that $x^\star$ is an optimal solution, we say $x$ is an $\eps$-optimal solution if and only if the optimality gap
	$
	F(x) - F(x^\star)
	$
	is no larger than $\eps$. Since the Frank-Wolfe gap is an upper bound on the optimality gap when $F$ is convex, $\calG(x^k)$ is a conservative measure of the optimality gap at $x^k$, see \cite{jaggi2013revisiting} for example. The Frank-Wolfe gap does not necessarily bound the optimality gap in the non-convex setting; nevertheless $\calG(x)$ is always nonnegative and $\calG(x) =0$ if and only if $x$ is a stationary point of \cref{pro: ERM with linear prediction}, and for this reason $\calG(x)$ is often used as a measure of non-stationarity at $x$, see \cite{lacoste2016convergence,reddi2016stochastic,yurtsever2019conditional,negiar2020stochastic}. For an instance of \cref{pro: ERM with linear prediction} with non-convex objective function $F$, we therefore say that a point $x \in \calC$ is an $\eps$-stationarity point of \cref{pro: ERM with linear prediction} if $\calG(x)\le \eps$.

	Generally speaking, for deterministic versions of the Frank-Wolfe method the lower bound LMO complexity to obtain an $\eps$-optimal solution is $O(1/\eps)$, see \cite{jaggi2013revisiting,lan2013complexity}. And in the case when $F$ is non-convex, the LMO complexity is at least $O(1/\eps^2)$, see \cite{shen2019complexities}. These deterministic methods require access to exact full-batch gradients per iteration, whose complexity scales with the number of observations $n$. When $n$ is very large, computing exact gradients might be no longer efficient. We refer the reader to \cite{frank1956algorithm,jaggi2013revisiting,lan2016conditional,lacoste2016convergence} among other recent deterministic versions of the Frank-Wolfe method in this context.
	
	Stochastic Frank-Wolfe methods avoid computing the exact gradient in each iteration by doing stochastic gradient estimation aimed at reducing the dependence on $n$ in the overall computational complexity. Most existing stochastic Frank-Wolfe methods typically replace exact gradients by the average of sampled gradients as recorded in the current or previous iterates, so that they can reduce (or eliminate) the dependence on $n$ in the overall complexity.  Unfortunately the approaches that eliminate the dependence on $n$ do so at the expense of a higher overall LMO complexity which tends to grow much faster than $O(1/\eps)$; see \cite{hazan2016variance,lan2016conditional,mokhtari2020stochastic} for representative work in this vein.
	A second approach is to estimate gradients by the variance reduction techniques which emerged from other stochastic first-order methods, such as SVRG \cite{johnson2013accelerating} and SPIDER \cite{fang2018spider}. These methods still periodically require access to the exact full gradient, and they also compute averages from a gradually increasing number of gradient samples, see \cite{hazan2016variance,reddi2016stochastic,yurtsever2019conditional,shen2019complexities,hassani2020stochastic}. Also unfortunately, some of these methods perform worse in practice than the methods with even weaker theoretical guarantees  \cite{hazan2016variance}.
	A third approach, that applies just to $\erml$, is based on the primal-dual structure of $\erml$ when $f$ is convex. 
	However, depending on the specific assumptions and problem set-up, the overall dependence of these methods on $\eps$ and $n$ is still $O(n/\eps)$, which is the same as the traditional (deterministic) Frank-Wolfe method, see \cite{lu2021generalized,negiar2020stochastic}.
	In addition to the above approaches, some methods avoid directly computing gradients by estimating them via zeroth-order information (see, e.g., \cite{gao2020can,huang2020accelerated,balasubramanian2018zeroth,sahu2019towards}). The theoretical complexities of these methods are usually no better than the methods using gradients. 
	Recently, a new line of research has emerged that seeks to estimate gradients with the aid of second-order information, as seen in \cite{shen2019complexities,zhang2020one,hassani2020stochastic}. However, the cost per iteration in these methods may still grow fast, or they still periodically require access to the exact gradient. In summary, there are no existing stochastic Frank-Wolfe methods (that we are aware of) whose overall complexity does not scale with $n$ while maintaining a dependence on $\eps $ of only $O(1/\eps)$ (for convex losses) or $O(1/\eps^2)$ (for non-convex losses).

	\subsection{Contributions}
	\begin{itemize}
		\item We develop a new family of Frank-Wolfe methods for $\erm$, which we call TUFW for ``Taylor-point Updating Frank-Wolfe,'' that replaces the exact gradient computation by a sum of (second-order) Taylor-approximated gradients around some current and previous iterates, which we call the \textit{Taylor points}. 
		\item Different versions of TUFW are developed based on different rules for constructing the batches of observations for Taylor-point updating at each iteration. For both convex and non-convex losses, we propose both stochastic and deterministic rules.  Our rules exhibit a {\em decreasing} number of gradient calls over the course of the intended iterations, while retaining the optimal LMO complexity of $O(1/\eps)$ (for convex losses) and $O(1/\eps^2)$ (for non-convex losses), together with other flops ($O(n/\sqrt{\eps})$ for convex losses and $O(n/\eps^{3/2})$ for non-convex losses.  In the regime when $\eps$ is sufficiently small these other flops are minor compared with the LMO complexity. Our stochastic TUFW also avoids periodically computing exact gradients. 
		\item We present computational experiments which show that our methods exhibit very significant speed-ups over existing methods on real-world datasets for both convex and non-convex binary classification problems. 
		\item We also propose a novel adaptive step-size method which has similar theoretical guarantees, and for which our computational experiments indicate superior performance in practice.
		
	\end{itemize}

	\subsection{Outline} In \cref{tufw} we present our new type of Frank-Wolfe method, which we call TUFW for ``Taylor-point Updating Frank-Wolfe'', along with some elementary properties of the method.  In \cref{sec: Convergence Guarantees Convex} we present computational guarantees and overall complexity of TUFW for convex losses, and in \cref{sec: Convergence Guarantees non-convex} we present computational guarantees and overall complexity of TUFW for nonconvex losses. 
	In \cref{sec erml problems} we specialize our results to the setting of $\erml$ problems.
	In \cref{sec: Extensions} we propose an adaptive step-size scheme. 
	Finally, \cref{exper} contains computational experiments for both convex and non-convex instances.
	
	\subsection{Notation}
	
	Let $\mathbb{N}$ denote the natural numbers.  For any $n \in \mathbb{N}$, we use $[n]$ to denote the set $\{1,2,\dots,n\}$.
	For any real number $v$, we use $\lfloor v \rfloor$ to denote the largest integer number less than or equal to $v$.
	Depending on the context, we will use either $\langle a,b\rangle$ or $a^\top b$ to denote the inner product between $a$ and $b$ in Euclidean space. For a vector $x$ in $\mathbb{R}^p$, the norm $\|x\|_q$ is given by $(\sum_{j=1}^{p}x_j^q)^{1/q}$ for $q \in [1,\infty)$ and $\|x\|_\infty := \max_{j} |x_j|$. The dual of norm $\|\cdot\|$ is defined as $\|z\|_* := \sup\{z^\top x : \|x\| \le 1\}$.  For a matrix $A \in \mathbb{R}^{p \times p}$, the matrix operator norm $\|A\|$ induced by the norm $\|\cdot\|$ is defined as $\sup_{x: \|x\| \le 1} \|Ax\|_*$. For any symmetric positive definite matrix $A\in \mathbb{S}^{p\times p}$, we define its induced norm on $\mathbb{R}^p$ by $\|z\|_A := \sqrt{z^\top A z}$.  We define $\sum_{i=j}^k \alpha_i := 0$ whenever $j > k$.  We denote the uniform distribution on the set $S$ by $\calU(S)$, and we denote the Bernoulli distribution with probability $p$ by $\mathrm{Ber}(p)$.

	\section{A Frank-Wolfe method with Taylor-point Updating (TUFW)}\label{tufw}
	
	In this section we present a new type of Frank-Wolfe algorithm whose newness derives from the way in which approximate gradients are computed.  We call the algorithm TUFW for ``Taylor-point Updating Frank-Wolfe''. Algorithm TUFW is similar to the standard Frank-Wolfe method except that it computes an inexpensive estimate $g^k$ of the gradient $\nabla F(x^k)$ at iteration $k$.  And it differs from other Frank-Wolfe methods in the way it computes the estimate $g^k$ in that (i) TUFW explicitly uses second-derivative (Hessian) information of the loss functions $f_i$, and (ii) it can use very different batch-size rules -- both stochastic and deterministic. TUFW is described in \cref{alg: TUFW in general finite sum}.  We now explain the steps of the method.

	\begin{algorithm}
		\caption{(TUFW) Frank-Wolfe Method with Taylor Point Updating}
		\label{alg: TUFW in general finite sum}
		\begin{algorithmic}[1]
			\STATE \textbf{input}  initial point $x^0$, step-size sequence $\{\gamma_k\}_k$, and (possibly) number of iterations $K$
			\STATE \textbf{initialize Taylor points} $\calB_0 \gets [n]$ and $b_i \leftarrow x^0$ for $i\in [n]$  \; \label{bettina}
			\FOR{$k=0,1,2,\dots,K$,}
			\STATE\label{algline: 3}\textbf{If $k>0$, update Taylor points}: \\ \ \ Use some \textbf{Rule} to construct set $\calB_k \subset [n]$ of indices for Taylor-point updating\\ \ \ Update Taylor points:  for $i\in\calB_k$ update $b_i \gets x^k$ \; \label{fanette}
			\STATE\label{algline: 4}  \textbf{Compute estimates of individual gradients}:\\ $g^k_{i}\gets \nabla f_{i}(b_{i})  +  \nabla^2 f_i (b_{i}) ( x^k -  b_{i} ) $ for $i\in [n]$\; \label{charlie}
			\STATE\label{algline: 5}   \textbf{Compute estimate of full gradient}: $g^k\gets \frac{1}{n}\sum_{i=1}^{n}g_i^k$\; \label{beverly}
			\STATE\label{algline: 6}   \textbf{Solve LMO}: $s^{k}\gets \arg\min_{s\in\calC} \langle g^k,s\rangle$\; \label{alex} 
			\STATE\label{algline: 7}   \textbf{Update feasible solution}: $x^{k+1}\gets x^k + \gamma_k (s^{k} - x^k)$. \label{hotel}
			\ENDFOR
			\RETURN $x^{K+1}$
		\end{algorithmic}
	\end{algorithm}

	Suppose we are at iteration $k$.  Then as a prerequisite to solving the LMO (Linear Minimization Oracle) problem in \cref{alex} of Algorithm TUFW, we first need to compute an (accurate) estimate $g^k$ of $\nabla F(x^k)$, which in turn means computing an (accurate) estimate $g_i^k$ of $\nabla f_i(x^k)$ for $i \in [n]$.  In an exact method we would simply compute $g_i^k \leftarrow \nabla f_i(x^k)$ and then compute $g^k\gets \frac{1}{n}\sum_{i=1}^{n}g_i^k$, the cost of which grows at least as $O(np)$ operations.  Here we instead consider, for each $i \in [n]$, computing the estimate $g_i^k$ of $\nabla f_i(x^k)$ by using the second-order Taylor approximation of $\nabla f_i(x)$ around some point $b_i$, resulting in the estimate:
	$$g^k_{i}\gets \nabla f_{i}(b_{i})  +  \nabla^2 f_i (b_{i}) ( x^k -  b_{i} )  \ \ \ \mathrm{for}\  i \in [n] \ , $$ 
	which is \cref{charlie} of TUFW.  We refer to $b_i$ as the ``Taylor point'' for the Taylor approximation of $\nabla f_{i}(\cdot)$, and $b_1, \ldots, b_n $ are the collection of Taylor points.  The Taylor points for different observations might be different, namely $b_i \ne b_j$.  Also, the Taylor point $b_i$ for observation $i$ will be {\em updated} from time to time over the course of the algorithm.  In Algorithm \ref{alg: TUFW in general finite sum} we initially set all Taylor points to be the initial point $x^0$, namely $b_i \leftarrow x^0$ for $i=1,2,\dots,n$ in \cref{bettina}.  Notice that there is a Taylor point $b_i$ for each observation index $i \in [n]$.  
	
	In order for the Taylor points $b_1, \ldots, b_n $ to yield sufficiently accurate gradient estimates $g^k_1, \ldots, g^k_n $ of $\nabla f_1(x^k), \ldots, \nabla f_n(x^k) $ at iteration $k$, we will need to update some/all of the Taylor points to the current feasible point iterate $x^k$ for some subset $\calB_k \subset [n]$ at iteration $k$ of Algorithm TUFW.  That is, we update $b_i \gets x^k$ for $i\in\calB_k$.  This update is done in \cref{fanette} of Algorithm TUFW.  
	
	The set $\calB_k$ is the set of indices whose Taylor points are updated to $x^k$ at iteration $k$, and $\calB_k$ can be determined by any kind of  {\bf Rule}.  As a trivial example, if one uses the rule $\calB_k = [n]$ for all $k$, then all Taylor points are updated to the current iterate $x^k$ at each iteration $k$, in which case TUFW specializes to an instantiation of the traditional Frank-Wolfe method with exact gradient computation.  As another example, one can use the rule that $\calB_k$ is a randomly chosen subset of the indices of cardinality, say, $50$, or $\lfloor n/10 \rfloor$, or perhaps is iteration dependent such as cardinality $\lfloor \ln(n) \sqrt{k} \rfloor$.  As a third example, one can use the rule that $\calB_k$ is chosen deterministically in a cyclic fashion with fixed batch-size $B$, so that $\calB_1 =  \{1, \ldots, B\}$, $\calB_2 = \{B+1, \ldots, 2B\}$, etc., with modular arithmetic so that  $\calB_k \subset [n]$.  Let $\beta_k := |\calB_k |$ be the cardinality of $\calB_k$.  As we will see in succeeding sections, we can attain nearly-optimal complexity by developing batch-size rules for which $\beta_k$ is a {\em decreasing} function of $k$, both in the convex and non-convex setting. 
	This decreasing batch-size property is in stark contrast to other versions of Frank-Wolfe that use sampling of the observations to compute gradient estimates, wherein the batch-sizes grow with $k$, see \cite{hazan2016variance,yurtsever2019conditional}.
	
	Going back to the description \cref{alg: TUFW in general finite sum}, the Taylor-approximated (at Taylor point $b_i$)  gradient of $f_i$ is computed for each observation $i \in [n]$ in \cref{charlie}, and the estimate of the full gradient is computed in \cref{beverly}.  Finally, the Linear Minimization Oracle (LMO) is called in \cref{alex} and the new iterate is computed in \cref{hotel} using the step-size $\gamma_k$ along the cord $[x^k, s^k]$.
	
	In the rest of the paper, we will presume that computing $\nabla f_i(\cdot)$ and $\nabla^2 f_i(\cdot)$ requires $O(p)$ and $O(p^2)$ flops respectively. (In the specific case of linear prediction in $\erml$, this is equivalent to presuming that computing $l_i'(\cdot)$ and $l_i''(\cdot)$ requires $O(1)$ flops.) Inspecting \cref{alg: TUFW in general finite sum}, it appears that the computation of the gradient estimate $g^k$ in \cref{charlie} and \cref{beverly} requires $O(np^2)$ flops. However, we now show how to do the computation of $g^k$ in these steps in only $O( | \calB_k| p^2)$ flops by efficiently using the computed value of $g^{k-1}$.

	\begin{proposition}\label{camille} Let $\bar\beta_k := | \calB_k| $.  Then the updating of the second-order Taylor approximation model requires $O(\bar{\beta}_{k} p^2)$ flops and the computation of $g^k$ can be done via a matrix-vector product, which takes at most  $O(p^2)$ flops.  
	\end{proposition}
	\begin{proof} [Proof of \cref{camille}]
		Consider iteration $k-1$ and let the Taylor points at this iteration be $b_1, \ldots, b_n$.  Let us write out the computation of $g^{k-1}$ as:
		\begin{small}
			$$\begin{aligned}
				g^{k-1} &= \tfrac{1}{n}\sum_{i=1}^n \nabla f_i(b_i)  + \tfrac{1}{n}\sum_{i=1}^n \nabla^2 f_i(b_i) (x^{k-1} - b_i) \\ 
				&= \tfrac{1}{n}\sum_{i=1}^n (\nabla f_i(b_i)  - \nabla^2 f_i(b_i)b_i) + \bigg[\tfrac{1}{n}\sum_{i=1}^n \nabla^2 f_i(b_i) \bigg] x^{k-1}  = q_{k-1} + H_{k-1} x^{k-1} \ , \end{aligned}$$
		\end{small}where $q_{k-1} = \tfrac{1}{n}\sum_{i=1}^n (\nabla f_i(b_i)  - \nabla^2 f_i(b_i)b_i)  $ and $H_{k-1} = \tfrac{1}{n}\sum_{i=1}^n \nabla^2 f_i(b_i)$.  
		We assume that we can store the vectors $b_i$ for $i \in [n]$, as well as the vector $q_{k-1}$ and Hessian matrix $H_{k-1}$.  
		
		At iteration $k$, let $\calB_k$ be the set of observation indices $i$ where we replace the Taylor point $b_i$ with $x^k$.  
		In order to compute $g^k$ we first update $q^{k-1}$ and $H_{k-1}$ to $q^{k}$ and $H_{k}$ as follows:
		\begin{small}
			\begin{equation}\label{eq: intermedia variables}
				\begin{aligned}
					q_k &= q_{k-1} + \tfrac{1}{n} \sum_{i \in \calB_k} \left(\nabla f_i(x^k) - \nabla f_i(b_i)  - \nabla^2 f_i(x^k)x^k + \nabla^2 f_i(b_i)b_i  \right)   \\
					H_k &= H_{k-1} + \tfrac{1}{n} \sum_{i \in \calB_k} \left(\nabla^2 f_i(x^k) - \nabla^2 f_i(b_i)   \right) \ .
				\end{aligned}
			\end{equation}
		\end{small}These two update computations require $O(\bar{\beta}_kp^2)$ flops. We then compute $g^k$ via $
		g^k = q_k + H_k x^k$, which uses one matrix-vector product and requires at most $O(p^2)$ flops. And last of all, we update the values of $b_i$ to $x^k$ for $i \in \calB_k$, which requires $O(\bar\beta_k p)$ flops. 
	\end{proof}

	Moreover, TUFW can be even more computationally efficient by utilizing sparsity properties. Since the Frank-Wolfe method often promotes sparse structured solutions \cite{jaggi2013revisiting}, the iterates $x^k$ are typically sparse in certain settings. Furthermore for problems such as matrix completion  \cite{fazel2002matrix, shen2019complexities}, the Hessian matrices are also highly sparse \cite{shen2019complexities}. Therefore matrix-vector products and the updating of $H_k$ can often be done with much fewer than $\calO(p^2)$ flops by using sparse linear algebra. 
	In addition, if the batches $\calB_k$ become empty with higher probability for large $k$, then $q_{k}$ and $H_k$ remain unchanged, and computing $g^k$ can be further simplified.

	\section{Convergence Guarantees for Convex Loss Functions}\label{sec: Convergence Guarantees Convex} 
	
	In this section we study convergence guarantees and overall complexity of \cref{alg: TUFW in general finite sum} for tackling the $\erm$ problem when $F$ is convex. 
	

	Recall that in order to run \cref{alg: TUFW in general finite sum} we need to specify the step-sizes $\{\gamma_k\}_k$ and the {\bf Rule} to construct the sets $\calB_k$ of indices for Taylor-point updating in \cref{fanette}.  Our first rule is designed 
	for $\calB_k$ to be comprised of $\beta_k := n /\sqrt{k}$ independently drawn samples from $[n]$ without replacement, for all $k \ge 1$.   Since $\beta_k$ is not in general an integer, then rather than ensuring $|\calB_k| = \beta_k$, we will instead ensure that $\mathbb{E}|\calB_k| = \beta_k$ by proceeding as follows.

	\begin{definition}\label{rule 1} $\rone$.  Define $\beta_k := n /\sqrt{k}$ and $p_k := \beta_k - \lfloor \beta_k \rfloor$.  Sample $\xi_k$ from the Bernoulli distribution $\mathrm{Ber}(p_k)$ and then uniformly sample $\lfloor \beta_k \rfloor + \xi_k$ samples from $[n]$ without replacement.  Then define $\calB_k$ to be the set of these index samples.\end{definition}
	
	Note that by design of $\rone$ it follows that $\mathbb{E}_{\xi_k} |\calB_k| = \beta_k$.  We call this rule ``$\rone$'' for ``Stochastic Batch-size Decreasing at the rate $\sqrt{k}$,'' and we emphasize that the Taylor points are updated less often as $k$ grows, which is in contrast to many other stochastic Frank-Wolfe methods that compute exact individual gradients more frequently as $k$ grows, see \cite{hazan2016variance,yurtsever2019conditional}. 
	
	We first present a bound on the total number of flops used in the first $k$ iterations of \cref{alg: TUFW in general finite sum} using $\rone$.  Let $\mathrm{fLMO}$ denote the number of flops required by the LMO in \cref{alex} of \cref{alg: TUFW in general finite sum}.

	\begin{proposition}\label{karljunior} Using $\rone$ and $k \ge 1$, the expected total number of flops used in the first $k$ iterations of \cref{alg: TUFW in general finite sum} is $O\left(k \cdot \left(\mathrm{fLMO} + p^2\right) + \sqrt{k} \cdot np^2\right)$.
	\end{proposition}
	
	\begin{proof}[Proof of \cref{karljunior}] For the initial iteration of \cref{alg: TUFW in general finite sum} the number of flops is $O(\mathrm{fLMO} + np^2)$, and the expected total number of flops in the first $k$ iterations is 
		\begin{small}
			$$
			\begin{aligned}
				O\left(k \cdot \mathrm{fLMO} + np^2 + \sum_{i=1}^k (\beta_i+1) p^2 \right) = O\left(k \cdot \mathrm{fLMO} + kp^2 + np^2 + \sum_{i=1}^k \frac{n p^2}{\sqrt{i}}\right) \\ 
				\le O\left(k \cdot \left(\mathrm{fLMO} + p^2\right) + np^2 + n p^2 \cdot 2 x^{\tfrac{1}{2}} \bigg|_{x=1}^{k+1}  \right) = O\left(k \cdot \left(\mathrm{fLMO} +  p^2\right) + n p^2 \sqrt{k} \right) \ ,
			\end{aligned}
			$$
		\end{small}where the left-most term follows from \cref{camille} and the inequality uses a standard integral bound.
	\end{proof}

	Let $D$ denote the diameter of $\calC$ under the norm of the variable space, namely
	$$
	D:= \max_{u,v\in\calC} \| u-v\| \ .
	$$
	The following theorem presents the computational guarantee for \cref{alg: TUFW in general finite sum} with $\rone$.  The proof of this theorem, as well as the other ensuing results in this sectioin, are presented in \cref{proof-convex-generalerm}.
	\begin{theorem}\label{thm: general convergence rate convex}
		Suppose that $F$ is convex and Assumption \ref{assump: ERM} holds, and  \cref{alg: TUFW in general finite sum} with $\rone$ is applied to the problem \cref{pro: ERM} with step-sizes defined by $\gamma_k:={2}/({k+2})$ for all $k \ge 0$. Then for all $k\ge 1$ we have:
		\begin{equation}\label{eq of thm: general convergence rate convex}
			\mathbb{E}[F(x^{k}) - F(x^\star)] \le  \frac{2L D^2 + 134 \hat{L}D^3}{k+1} \ .
		\end{equation} 
	\end{theorem} 
	
	\noindent Using \cref{thm: general convergence rate convex} together with \cref{karljunior}, we obtain the following bound on the overall number of required flops to compute an $\varepsilon$-optimal solution in expectation. 
	
	\begin{corollary} \label{cor general convex erm}
		Under the hypotheses of \cref{thm: general convergence rate convex}, it holds that $\mathbb{E}[F(x^{k}) - F(x^\star) ] \le \varepsilon$ after at most $$ \frac{2L D^2 + 134 \hat{L}D^{3}}{\varepsilon} $$
		iterations, and the total number of required flops is at most
		\begin{small}
			\begin{equation*} O\left( (\mathrm{fLMO} + p^2) \left[\frac{L D^2 + \hat{L}D^3}{ \varepsilon} \right] + np^2
				\left[\frac{\sqrt{L D^2 + \hat{L}D^3}}{\sqrt{\varepsilon}} \right] 
				\right) \ .  \end{equation*}
		\end{small}
	\end{corollary}

	Whereas $\rone$ is stochastic, we now present the following deterministic rule which achieves nearly identical computational guarantees (to within a constant factor) but whose computational guarantees are deterministic. 
	
	\begin{definition}\label{rule 3} $\rthree$.  For any $k\ge 1$ define 
		$$
			\calB_k = \left\{
			\begin{array}{cl}
				[n] & \text{ if  } \ \sqrt{k}\in \mathbb{N}  \\
				\emptyset & \text{ if  } \ \sqrt{k}\notin \mathbb{N} \ .
			\end{array}
			\right.
			$$
	\end{definition}
	
	In $\rthree$ we do not update any Taylor points unless $k = 1, 4, 9, 25, \ldots$, and for these values of $k$ we update all $n$ Taylor points.  We call this approach ``$\rthree$'' for ``Deterministic Batch-frequency Decreasing at the rate $\sqrt{k}$,'' and we point out that for $\rthree$ the Taylor points are updated less often as $k$ grows (in a different way but with similar effect as in $\rone$).

	Similar to the case of $\rone$, we have:
	
	\begin{proposition}\label{karlsenior} Using $\rthree$ and $k \ge 1$, the total number of flops used in the first $k$ iterations of \cref{alg: TUFW in general finite sum} is $O(k \cdot (\mathrm{fLMO}+p^2) + \sqrt{k} \cdot np^2)$.
	\end{proposition}
	
	\begin{proof}[Proof of \cref{karlsenior}] The proof is almost identical to that of \cref{karljunior}. Let $\beta_k :=|\calB_k|$. For the initial iteration of Algorithm TUFW the number of flops is $O(\mathrm{fLMO} + np^2)$, and the number of flops in the first $k$ iterations is 
		\begin{small}
			$$\begin{array}{rcl}O\left(k \cdot \mathrm{fLMO} + np^2 + \sum_{i=1}^k (\beta_i+1) p^2 \right) & \le & O\left(k \cdot (\mathrm{fLMO} +p^2) +  n p^2 \sqrt{k} \right) \ , \end{array}$$
		\end{small}where the left-hand side follows from \cref{camille} and the inequality follows since $\sum_{i=1}^k \beta_i \le n \sqrt{k}  $. \end{proof}

	\begin{theorem}\label{thm:  general convergence rate convex rule 3}
		Suppose that $F$ is convex and \cref{assump: ERM} holds, and \cref{alg: TUFW in general finite sum} with $\rthree$ is applied to the problem \cref{pro: ERM} with step-sizes defined by $\gamma_k:={2}/({k+2})$ for all $k \ge 0$. Then for all $k\ge 1$ we have:
		\begin{equation}\label{eq: thm: general convergence rate convex rule 3}
			F(x^{k}) - F(x^\star) \le  \frac{2L D^2+ 144\hat{L}D^3}{k+1} \ . 
		\end{equation} 
	\end{theorem} 
	
	\noindent Using \cref{thm:  general convergence rate convex rule 3} together with \cref{karlsenior}, we obtain the following bound on the overall number of required flops to compute an $\varepsilon$-optimal solution. 
	
	\begin{corollary} \label{lorisenior general}
		Under the hypotheses of \cref{thm:  general convergence rate convex rule 3}, it holds that $F(x^{k}) - F(x^\star) \le \varepsilon$ after at most $$ \frac{2L D^2 + 144 \hat{L}D^3}{\varepsilon}$$ iterations, and the total number of flops required is at most
		\begin{small}
			\begin{equation*}O\left( (\mathrm{fLMO}+p^2) \left[\frac{L D^2 + \hat{L}D^3}{ \varepsilon} \right] + np^2
				\left[\frac{\sqrt{L D^2 + \hat{L}D^3}}{\sqrt{\varepsilon}} \right]
				\right) \ . \end{equation*}
		\end{small}
	\end{corollary} 
	
	We now seek to compare the computational guarantees of our algorithm TUFW with other relevant Frank-Wolfe methods developed for $\erm$ and its more specialized form $\erml$.  Because of the prevalance of problem instances with linear prediction, we first need a convenient way to translate the bounds for TUFW for $\erm$ into bounds for $\erml$ (which assumes the loss functions have the property of linear prediction). We do so as follows. 
	
	\begin{proposition}\label{lm transfer assumption} In the setting of $\erml$, suppose that all feature vectors $w_i$ for $i=1,2, \ldots $ lie in a bounded set $S \subset \{ w \in \mathrm{R}^p : \|w\|_* \le M\}$ regardless of the number of observations $n$, where $M$ is the radius of the (dual norm) ball centered at the origin in $\mathbb{R}^p$ containing $S$.  If an instance of $\erml$ satisifies \cref{assump: ERM with linear prediction}, then $\nabla f_i(\cdot)$ is $LM^2$-Lipschitz continuous and $\nabla^2 f_i(\cdot)$ is $LM^3$-Lipschitz continuous, whereby the instance satisfies \cref{assump: ERM} with these Lipschitz constants.
	\end{proposition}
	\begin{proof}[Proof of \cref{lm transfer assumption}]
		For  any $u$, $v$ in $\calC$ and any $i\in[n]$, it holds that 
		\begin{small}
			\begin{equation}\label{eq continuous f}
				\begin{aligned}
					& \|\nabla f_i(u) - \nabla f_i(v)\|_*   = \| l_i'(w_i^\top u) w_i - l_i'(w_i^\top v) w_i \|_* = | l_i'(w_i^\top u)  - l_i'(w_i^\top v) | \cdot \|w_i\|_* \\
					\le \ & L\cdot | w_i^\top u  - w_i^\top v | \cdot \|w_i\|_* \le L\cdot \| u  - v \| \cdot \|w_i\|_*^2 \le L  M^2 \cdot \| u  - v \|  \ , 
				\end{aligned}
			\end{equation}
		\end{small}where the first inequality follows from  \cref{assump: ERM with linear prediction} for $\erml$ and the second inequality uses $\|w_i\|_* \le M$.  This shows that that $\nabla f_i(\cdot)$ is $LM^2$-Lipschitz continuous. Similarily, we have
		\begin{small}\begin{equation}\label{eq continuous f hessian}
			\begin{aligned}
				& \|\nabla^2 f_i(u) - \nabla^2 f_i(v)\|  = \big\| \big(l_i''(w_i^\top u) - l_i''(w_i^\top v)\big) \cdot w_i w_i^\top  \big\| \\
				\le \ & \hat{L}\cdot | w_i^\top u  - w_i^\top v | \cdot \big\|w_i w_i^\top  \big\| \le \hat{L} \cdot \| u  - v \| \cdot \|w_i\|_*^3 \le \hat{L}  M^3 \cdot \| u  - v \|  \ ,
			\end{aligned}
		\end{equation}\end{small}which shows that $\nabla^2 f_i(\cdot)$ is $\hat{L}M^3$-Lipschitz continuous.
	\end{proof}

	We now use \cref{lm transfer assumption} to constructively compare several different Frank-Wolfe methods for both $\erm$ and $\erml$. \cref{tbl: compare rules convex} shows a comparison of the computational guarantees of the standard Frank-Wolfe method, the Constant batch-size Stochastic Frank-Wolfe (CSFW) method developed in \cite{negiar2020stochastic}, and the TUFW method with $\rone$ and $\rthree$. 
	
	\begin{table}[htbp]
		\begin{centering}
			\caption{Complexity bounds for different Frank-Wolfe methods to obtain an $\eps$-optimal solution of $\erm$ and $\erml$ with convex losses.  In the table $x^\star$ is an optimal solution and $\eps_0:=F(x^0) - F(x^\star)$.}\label{tbl: compare rules convex}
		\end{centering}
		\renewcommand\arraystretch{2.1} 
		\begin{adjustbox}{width=1\columnwidth,center}
			\begin{tabular}{c|c|c}
				\multicolumn{1}{c}{Method} & \multicolumn{1}{c}{Optimality Metric} & \multicolumn{1}{c}{Overall Complexity} \\
				\hline \hline
				\multicolumn{3}{l}{\small $\erm$ setting under \cref{assump: ERM}.  Here $c_1:=L D^2$,  and $c_2:=\hat{L}D^3$ } \\
				\hline
				$\rone$ (Corollary $\ref{cor general convex erm}$)          & $\mathbb{E}[F(x^{k}) - F(x^\star) ]  \le \varepsilon$ & $\displaystyle O\left( (\mathrm{fLMO}+p^2) \cdot \frac{c_1+c_2}{\eps}  + np^2 \cdot \frac{\sqrt{c_1 + c_2}}{\sqrt{\eps}} \right)$                                                   \\  
				$\rthree$ (Corollary $\ref{lorisenior general}$)         & $F(x^{k}) - F(x^\star)  \le \varepsilon$              & $\displaystyle O\left( (\mathrm{fLMO}+p^2) \cdot \frac{c_1+c_2}{\eps}  + np^2 \cdot \frac{\sqrt{c_1 + c_2}}{\sqrt{\eps}} \right)$   \\  
				standard Frank-Wolfe  ($\cite{jaggi2013revisiting}$)  & $F(x^{k}) - F(x^\star) \le \varepsilon$               & $\displaystyle O\left( (\mathrm{fLMO}+np) \cdot \frac{c_1}{\eps}\right)$ 
				\\
				\hline  \multicolumn{3}{l}{\small $\erml$ setting under  \cref{assump: ERM with linear prediction}.  Here $c_1:=L M^2D^2$,  and $c_2:=\hat{L}M^3D^3$ } \\
				\hline 
				$\rone$ and $\rthree$ &\multicolumn{2}{c}{Same as the corresponding rows in the $\erm$ setting above}
				\\  
				CSFW ($\cite{negiar2020stochastic}$) & $\mathbb{E}[F(x^{k}) - F(x^\star) ]  \le \varepsilon$ & $O\displaystyle\left((n\cdot \mathrm{fLMO} + np) \left[\frac{c_1}{\eps} + \frac{ \sqrt{\eps_0} }{n\sqrt{\eps}} + \frac{ \sqrt{n c_1} }{\sqrt{\eps}} \right]    \right)$                                         
			\end{tabular}
		\end{adjustbox}
	\end{table}
	

	Note in particular that with $\rone$ or $\rthree$, the TUFW's joint dependence on $n$ and $\varepsilon$ is $O(n/\sqrt{\varepsilon})$, as compared to $O(n/\varepsilon)$ in \cite{jaggi2013revisiting,negiar2020stochastic,lu2021generalized}, $O(1/\eps^2)$ in \cite{hazan2016variance,yurtsever2019conditional,zhang2020one,hassani2020stochastic}, or $O(1/\varepsilon^3)$ in \cite{mokhtari2020stochastic}. 
	Indeed, in the regime where $\eps$ is sufficiently small, the dominant term in the TUFW's complexity bound is the left term (which counts operations of the LMO), which is independent of $n$.
	In this regime the overall complexity is nearly optimal: the LMO complexity's dependence on $\eps$ is $O(1/\eps)$, which is essentially the same as in the lower bound complexity result in \cite{lan2013complexity}. The LMO complexity of TUFW differs from the lower bound only by the term involving $\hat{L}$, which is the Lipschitz constant of the second-order derivatives. 
	It should also be mentioned that  the TUFW's joint dependence on $n$, $p$ and $\eps$ is $O(p^2/\eps + np^2/\sqrt{\eps})$, as compared to $O(n p/\eps)$ in \cite{jaggi2013revisiting}.  As long as $n$ is large enough compared with $p$, i.e. $n > p \cdot (c_1 + c_2) / c_1 $, and $\eps$ is sufficiently small, TUFW outperforms the standard Frank-Wolfe method  in terms of the dependence on $n$ and $p$. For $\erm$ problems in practice, such as those in the examples in \cref{introsec}, $n$ can be significantly larger than $p$ and in this regime the upper bound complexity of TUFW dominates that of the standard Frank-Wolfe method.
	However, in the case where $p$ is significantly larger than $n$, the complexity bound for TUFW can be inferior to that of  the standard Frank-Wolfe method due to the second-order dependence on $p$.
	
	Note that the boundness requirement of $\|w_i\|_*$ can also be replaced with the boundness of the linear predictions $W^\top(\calC)$, as is done in \cite{negiar2020stochastic,lu2021generalized}. The computational guarantees for the $\erml$ in this boundness setting will be discusssed in \cref{sec erml problems}.

	Last of all, we consider the following rule which does no updates of the Taylor points beyond their initialization in \cref{bettina} of TUFW.
	
	\begin{definition}\label{rule 5} $\rfive$.  Define $\calB_0 =[n]$ and $\calB_k = \emptyset$ for all $k \ge 1$.\end{definition}
	
	$\rfive$ is quite relevant in the special case when the loss functions $f_i$ are all quadratic loss functions (as in linear regression, matrix completion, LASSO, etc.) in which case the Hessian $\nabla^2 F(x)=\nabla^2 F(x^0)$ for any $x$ and thus the Hessian does not change. Therefore in this case $g^k_i = \nabla f_i(x^k)$ and $g^k =\nabla F(x^k)$ for all $k$. When using $\rfive$, the computational cost of computing the initial individual gradients and the Hessian $\nabla^2 F(x^0)$ is only $O(np^2)$ flops, and the cost of calculating the (exact) gradient is the cost of a matrix-vector product, which is at most $O(p^2)$ flops per iteration. 
	
	\subsection{Proofs of results}\label{proof-convex-generalerm}
	We will prove \cref{thm: general convergence rate convex} as a consequence of the following sequences of results. 
	In the rest of the paper, we will use $x^\star$ to denote an optimal solution of the $\erm$ problem, and let $\eps_k := F(x^k) - F(x^\star)$. 
	
	We begin with the following straightforward properties of smooth functions.
	\begin{proposition}\label{prop: general basic ineq}
		Under \cref{assump: ERM} for $\erm$,  for any $x$, $y\in\calC$ and $i$, 
		\begin{align}
			| f_i(y) -  f_i(x) - \langle\nabla f_i(x),  y- x\rangle |\le  L\|x-y\|^2 /2 \ , \text{ and } \label{eq: fact: general basic ineq 1} \\
			\| \nabla f_i(y) - \nabla f_i(x) - \langle\nabla^2 f_i(x),  y- x\rangle \|_*\le  \hat{L}\|x-y\|^2 /2 \ . \label{eq: fact: general basic ineq 2}
		\end{align}
	\end{proposition}
	
	For any $k \ge 0$ and $i\in [n]$, we will be interested in the most recent iteration (up through $k$) at which Taylor point $b_i$ was updated.  This is defined as:
	\begin{equation}\label{def: tau}
		\tau^k_i := \max\{t: t\le k  \text{ and } i \in \calB_t\} \ ,
	\end{equation}
	which means that in the first $k$ iterations, $\tau^k_i$ is the latest iteration in which the Taylor point $b_i$ was updated. In \cref{alg: TUFW in general finite sum}, the relationship between $\tau^k_i$ and $\tau^{k-1}_i$ is then:
	\begin{equation}\label{eq: tauk relationship}
		\tau^{k}_i = \left\{
		\begin{array}{ll}
			k & \text{ if } \  i\in\calB_k \\
			\tau^{k-1}_i & \text{ if } \ i\notin\calB_k \ . 
		\end{array}
		\right.
	\end{equation}

	The following lemma will be useful in measuring the error of Taylor-estimated gradients.
	\begin{lemma}\label{lm: general 3rd term upper bound}
		Under \cref{assump: ERM} for the problem $\erm$, for any $u,v \in\calC$,
		\begin{equation}\label{eq: lm: general 3rd term upper bound}
			\left(
			\nabla F(x^k)   - g^k
			\right)^\top (u - v) \le \frac{\hat{L}D^3}{2n}\sum_{i=1}^{n}\Big( \sum_{j = \tau^k_i}^{k-1}  \gamma_j \Big)^2
		\end{equation}
		holds for iteration $k$ of Algorithm \ref{alg: TUFW in general finite sum} for all $k \ge 1$.
	\end{lemma}

	\begin{proof}[Proof of \cref{lm: general 3rd term upper bound}]
		First of all, $\left(
		\nabla F(x^k)   - g^k
		\right)^\top (u - v)$ can be rewritten as
		\begin{small}
			\begin{equation}\label{ineq: general etimate gradient 0}
				\frac{1}{n}\sum_{i=1}^{n}\left(\nabla f_{i}\big(x^k \big) -   \nabla f_{i}\big( x^{\tau^{k}_i}\big) -  \nabla^2 f_{i} \big( x^{\tau^{k}_i}\big)\big(x^k  - x^{\tau^{k}_i} \big)    \right)^\top (u - v),
			\end{equation}	
		\end{small}which, since $\|u-v\| \le D$, is smaller than or equal to 
		\begin{small}\begin{equation}\label{ineq: general etimate gradient 1}
			\frac{D}{n}\sum_{i=1}^{n}\left\| \nabla f_{i}\big(x^k \big) -   \nabla f_{i}\big( x^{\tau^{k}_i}\big) -  \nabla^2 f_{i} \big( x^{\tau^{k}_i}\big)\big(x^k  - x^{\tau^{k}_i} \big)    \right\|_* \le    \frac{\hat{L}D}{2n}\sum_{i=1}^{n}\big\|x^k  - x^{\tau^{k}_i}\big\|^2,
		\end{equation}\end{small}where the inequality is due to \cref{eq: fact: general basic ineq 2} in \cref{prop: general basic ineq}.
		Additionally, we have 
		\begin{small}\begin{equation*}
			\|x^k  - x^{\tau^{k}_i}
			\|^2 \le  \bigg( \sum_{j = \tau^k_i}^{k-1} \| x^{j+1} -  x^{j}\| \bigg)^2\le  \bigg( \sum_{j = \tau^k_i}^{k-1}  \gamma_j D \bigg)^2 \  ,
		\end{equation*}\end{small}and substituting this last bound into \cref{ineq: general etimate gradient 1} yields \cref{eq: lm: general 3rd term upper bound}.
	\end{proof}
	
	Let $\varepsilon_k := F(x^k) - F(x^\star)$ denote the optimality gap at iteration $k$. The following lemma provides an upper bound on $\eps_k$.
	\begin{lemma}\label{lm: general convergence rate convex}
		Under \cref{assump: ERM} for problem $\erm$, suppose that $F$ is conevx, and let the step-sizes be given by $\gamma_k=2/(k+2)$ for all $k \ge 0$. Then for all $k \ge 1$ we have
		\begin{equation}\label{eq lm: general convergence rate convex}
			\eps_k  \le  \frac{2L D^2}{k+1}
			+  \frac{2}{k(k+1)}\sum_{t=1}^{k} t \left(
			\nabla F(x^{t-1})   - g^{t-1}
			\right)^\top (s^{t-1} - x^\star) \ .
		\end{equation}
	\end{lemma}
	\begin{proof}[Proof of \cref{lm: general convergence rate convex}]
		Utilizing  \cref{prop: general basic ineq} we have
		\begin{small}\begin{equation}\label{eq lm general convergence rate convex 0}
				\begin{aligned}
					F(x^{k+1})\leq ~ & F(x^k)+\gamma_{k}\langle\nabla F(x^k),s^{k}- x^k \rangle+ \gamma_{k}^{2} {L D^2}/{2} \\
					= ~ & F(x^{k})+\gamma_{k}\langle \nabla F(x^{k})-g^k, s^{k}-x^k\rangle +\gamma_{k}\langle{g^k}, s^{k}-x^k\rangle+ \gamma_{k}^{2} {L D^2}/{2} \\
					\leq ~ & F(x^{k})+\gamma_{k}\langle\nabla F(x^{k})-{g^k}, s^{k}-x^k\rangle + \gamma_{k}\langle{g^k}, {x}^{\star}-x^k\rangle+ \gamma_{k}^{2} {L D^2}/{2} \\
					= ~ & F(x^{k})+\gamma_{k}\langle \nabla F(x^{k})-{g^k},s^{k}-{x}^{\star}\rangle +\gamma_{k}\langle \nabla F(x^{k}), {x}^{\star}-x^k\rangle+ \gamma_{k}^{2} {L D^2}/{2}\\
					\le ~& F(x^{k})+\gamma_{k}\langle \nabla F(x^{k})-{g^k},s^{k}-{x}^{\star}\rangle +\gamma_{k}(F(x^\star) - F(x^k))+ \gamma_{k}^{2} {L D^2}/{2} \ , 
				\end{aligned}
		\end{equation}\end{small}where the second inequality uses $s^{k} \in \arg\min_{s\in\calC} (g^k)^\top s$ from \cref{alex} of the algorithm, and the third inequality uses the gradient inequality applied to $F$.  Subtracting $F^\star$ from both sides of the above inequality chain, we arrive at:
		\begin{small}\begin{equation}\label{eq lm general convergence rate convex 1}
				\begin{aligned}
					\eps_{k+1}& \le (1-\gamma_k)\eps_k + \gamma_k \big(
					\nabla F(x^k)   - g^k
					\big)^\top (s^{k} - x^\star) +  {\gamma_k^2 L D^2}/{2} \ .
			\end{aligned}
		\end{equation}\end{small}Multiplying both side by $(k+1)(k+2)$ and telescoping the inequalities yields:
		\begin{small}\begin{equation}
				\begin{aligned}
					(k+1)(k+2)\eps_{k+1}  &\le  2(k+1)L D^2
					+  \sum_{t=0}^{k}2(t+1) (
					\nabla F(x^t)   - g^t
					)^\top (s^{t} - x^\star) \ ,
			\end{aligned}
		\end{equation}\end{small}which yields \cref{eq lm: general convergence rate convex} after rearranging and re-indexing the counter $t$.
	\end{proof}
	
	To study the upper bound on $\eps_k$ in \cref{lm: general convergence rate convex} a bit further, we will use the following lemma. 
	
	\begin{lemma}\label{lm: convex EZ} Let the series $\{d_{ij}\}_{i\in[n],j\in\mathbb{N}}$ be given, and suppose that there exist $M_{1}$ and $M_{\infty}$ for which $\sup_{j} \sum_{i=1}^n d_{ij} \le M_{1}$ and  $\sup_{i,j} d_{ij}\le M_{\infty}$.  
		Let $\tau_i^k$ be defined in \cref{def: tau}, in which $\calB_t$ is constructed using $\rone$ in \cref{rule 1}. Let $\gamma_k$ be ${2}/(k+2)$ for all $k \ge 0$.  Then it holds that
		\begin{equation}\label{eq: lm: convex EZ}
			\mathbb{E}\bigg[\sum_{i=1}^{n}\Big( \sum_{j = \tau^k_i}^{k-1}  \gamma_j d_{ij} \Big)\bigg] \le
			\frac{6 M_1}{\sqrt{k+2}} \ \ \text{ and }   \ \
			\mathbb{E} \bigg[\sum_{i=1}^{n}\Big( \sum_{j = \tau^k_i}^{k-1}  \gamma_j d_{ij} \Big)^2 \bigg]\le  \frac{134M_1M_\infty}{k+2} \ .
		\end{equation}
	\end{lemma}
	
	Before proving \cref{lm: convex EZ}, we first present two useful results, \cref{lm: rho} and  \cref{lm: simple ineq}.
	
	\begin{proposition}\label{lm: rho}
		For integers $v$ and $k$ satisfying $1\le  v \le k$ it holds that:
		\begin{equation}\label{eq: fact 2}
			\sum_{i=1}^{v}\prod_{j=i}^{k}\left(1-\tfrac{1}{\sqrt{j}}\right) <
			(\sqrt{v+1})\exp(2\sqrt{v+1} - 2\sqrt{k+1})  \   .
		\end{equation}
	\end{proposition}
	\begin{proof}[Proof of \cref{lm: rho}]
		For convenience let us define $\rho_j := 1-\tfrac{1}{\sqrt{j}}$ for $j \ge 1$, and define the function $h(x):= \exp(2 \sqrt{x} - 2 \sqrt{k+1}) $ for $x \ge 0$.  We will first prove the following inequality:
		\begin{small}\begin{equation}\label{lulu}
			\prod_{j=i}^{k}\left(1-\tfrac{1}{\sqrt{j}}\right) \le h(i) \ \ \ \mathrm{for~all~} i \in [k] \ . 
		\end{equation}\end{small}For $i=1$ the left side above is $0$ (because $\rho_1 = 0$ and hence $\prod_{j=1}^{k}\rho_j = 0$), and the right side is $h(1) >0$, so \cref{lulu} is true for $i=1$. For $i\ge2$ we have:
		\begin{small}\begin{equation*}\begin{aligned}
				\prod_{j=i}^{k}\rho_j  &= \exp\left({\sum_{j=i}^{k}\ln(1-\tfrac{1}{\sqrt{j}})}\right) \le \exp\left( -\sum_{j=i}^{k}\tfrac{1}{\sqrt{j}}\right)  \\  
				&\le \exp\Big( -\int_i^{k+1} x^{-\tfrac{1}{2}} dx\Big) =\exp(2\sqrt{i} - 2\sqrt{k+1}) = h(i) \ , \end{aligned} 
		\end{equation*}\end{small}where the first inequality follows from the fact that $\log(1-x)\le -x$ for any $0\le x <1$, and the second inequality is an integral bound. This establishes \cref{lulu}, whereby we have
		\begin{small}$$ \begin{aligned}
				\sum_{i=1}^{v}   \prod_{j=i}^{k}\rho_j  &\le  \sum_{i=1}^{v} h(i) \le \int_{1}^{v+1}h(x)~dx = \left. (\sqrt{x}-\tfrac{1}{2})\exp(2\sqrt{x} - 2\sqrt{k+1}) \right|_1^{v+1} \\ 
				  &= (\sqrt{v+1}-\tfrac{1}{2})\exp(2\sqrt{v+1} - 2\sqrt{k+1}) - (\tfrac{1}{2})\exp(2 - 2\sqrt{k+1}) \\\
				&< (\sqrt{v+1}-\tfrac{1}{2})\exp(2\sqrt{v+1} - 2\sqrt{k+1}) \ ,
		\end{aligned}$$\end{small}where the second inequality is an integral bound using the fact that $h(x)$ is an increasing function. \end{proof}
	
	\cref{lm: simple ineq} presents an upper bound on the right-hand side of \cref{eq: fact 2} when $v = \lfloor k/2\rfloor$.
	
	\begin{lemma}\label{lm: simple ineq}
		For any $k\in\mathbb{N}$ the following inequality holds:
		\begin{equation}\label{ineq: lm: simple ineq} 
			\sqrt{\lfloor k/2 \rfloor + 1} \exp(2\sqrt{\lfloor k/2\rfloor + 1} - 2\sqrt{k+1}) \le \min\{{3}(k+2)^{-\frac{1}{2}},{10}(k+2)^{-1}\}. 
		\end{equation}
	\end{lemma}
	\begin{proof}[Proof of \cref{lm: simple ineq}]
		Since $h(x):=\sqrt{x}$ is a concave function on $\mathbb{R}_+$, we have $h(k+2) \le h(k+1) + h'(k+1)$, which yields $\sqrt{k+1}\ge \sqrt{k+2} - \frac{1}{2\sqrt{k+1}} \ge \sqrt{k+2} - \sqrt{2}/4 $ when $k \ge 1$.
		Furthermore $\sqrt{\lfloor k/2 \rfloor + 1} \le \frac{\sqrt{2}}{2} \sqrt{k+2}$, and therefore 
		\begin{small}\begin{equation*}
			\sqrt{\lfloor k/2 \rfloor + 1} \exp(2\sqrt{\lfloor k/2\rfloor + 1} - 2\sqrt{k+1}) \le  \frac{\sqrt{2}e^{\frac{\sqrt{2}}{2}}}{2}(k+2)^{\frac{1}{2}}  \exp(-(2-\sqrt{2})(k+2)^{\frac{1}{2}}) \ .
		\end{equation*}\end{small}Next note that the function $h(x):=e^{-x} x^t$ on the domain $x \in (0,\infty)$ for $t > 0$ attains its maximal value at $x = t$. In particular for $t=2$ and $t=3$, respectively, we conclude that $e^{-x} \le 4e^{-2}x^{-2}$ and $e^{-x} \le 27e^{-3}x^{-3}$ (respectively), and thus  $\exp(-(2-\sqrt{2})(k+2)^{\frac{1}{2}}) \le 4 e^{-2}(2-\sqrt{2})^{-2}(k+2)^{-1}$
		and $\exp(-(2-\sqrt{2})(k+2)^{\frac{1}{2}}) \le 27 e^{-3}(2-\sqrt{2})^{-3}(k+2)^{-3/2}$ (respectively).
		Therefore, the left-hand side of \cref{ineq: lm: simple ineq} is bounded by 
		$2\sqrt{2}e^{\frac{\sqrt{2}}{2}-2} (2-\sqrt{2})^{-2} (k+2)^{-\frac{1}{2}}$ and 
		$\frac{27}{2}\sqrt{2}e^{\frac{\sqrt{2}}{2}-3} (2-\sqrt{2})^{-3} (k+2)^{-1}$, respectively, which is less than $3(k+2)^{-\frac{1}{2}}$ and $10(k+2)^{-1}$, respectively.
	\end{proof}
	
	We now are in position to prove \cref{lm: convex EZ}.
	\begin{proof}[Proof of \cref{lm: convex EZ}]
		
		First of all, let us define $G_k^i := \left( \sum_{j = \tau^k_i}^{k-1}  \gamma_j d_{ij} \right)$ and $Z_k^i :=\left( \sum_{j = \tau^k_i}^{k-1}  \gamma_j d_{ij} \right)^2$ for $i \in [n]$, and also $G_k :=\sum_{i=1}^n G_k^i$ and $Z_k :=\sum_{i=1}^n Z_k^i$. Notice in $G_k^i$ and $Z_k^i$ that the randomness lies only in $\tau_i^k$, which in turn depends on the update sets $\calB_0, \calB_1,\dots,\calB_k$. 
		
		We first analyze $G_k^i$ and $G_k$.  We note that $G_k^i = 0$ when $i \in \calB_k$, which happens with probability $\beta_k/n$.  Because of the relationship between $\{\tau^k_i\}$ and $\{\tau^{k-1}_i\}$ in \cref{eq: tauk relationship}, we further note that with probability $1-\beta_k/n$ that $i \notin \calB_k$ in which case it follows that $G_k^i = G_{k-1}^i + \gamma_{k-1}d_{i,k-1}$.  Defining $\rho_k := 1 - \frac{1}{\sqrt{k}} = 1 - \beta_k/n $, we have 
		\begin{small}\begin{equation}\label{ncl}
			\mathbb{E}[ G_k ] = \sum_{i=1}^n \mathbb{E} [G_k^i ]= \rho_k \sum_{i=1}^n \mathbb{E} [G_{k-1}^i +
			\gamma_{k-1} d_{i,k-1}] \le \rho_k\left( \mathbb{E} [G_{k-1} ]+
			\gamma_{k-1} M_1 \right) \ ,
		\end{equation}\end{small}where the inequality uses $\sup_{j} \sum_{i=1}^n d_{ij} \le M_{1}$. Recurrently applying \cref{ncl} yields:
		\begin{small}\begin{equation}\label{ineq: recurrent G_k 2}
			 \mathbb{E}[ G_k ] \le \rho_k( \mathbb{E}[ G_{k-1} ]+
			\gamma_{k-1}  M_1 ) \le\cdots\le \sum_{i=1}^{k}\Big(\prod_{j=i}^{k}\rho_j\Big) \gamma_{i-1} M_1 + \Big(\prod_{j=1}^{k} \rho_j\Big)\mathbb{E}[ G_0] \ .
		\end{equation}\end{small}Noting that $G_0 = 0$ (since all Taylor points are set to $x^0$ in \cref{bettina} of \cref{alg: TUFW in general finite sum}) and using $\gamma_{i-1} = \frac{2}{i+1}$ it follows that:
		\begin{small}\begin{equation}\label{ineq: recurrent G_k 3}
			\begin{aligned}
				\mathbb{E} [G_k] &\le \sum_{i=1}^{\lfloor {k}/{2}\rfloor}\Big(\prod_{j=i}^{k}\rho_j\Big) \frac{2M_1}{i+1}
				+
				\sum_{i=\lfloor {k}/{2}\rfloor+1}^{k}\Big(\prod_{j=i}^{k}\rho_j\Big) \frac{2M_1}{i+1}\\
				& \le \sum_{i=2}^{\lfloor {k}/{2}\rfloor}\Big(\prod_{j=i}^{k}\rho_j\Big) \frac{2M_1}{3}
				+
				\sum_{i=\lfloor {k}/{2}\rfloor+1}^{k}\Big(\prod_{j=i}^{k}\rho_j\Big) \frac{2M_1}{\lfloor {k}/{2}\rfloor+2} \ ,
			\end{aligned}
		\end{equation}\end{small}where the second inequality uses the fact that $\rho_1 = 0$ and hence $\prod_{j=1}^{k}\rho_j = 0$.
		By \cref{lm: rho} we have:
		\begin{small}\begin{equation}\label{ineq: G_k recurrent sumpro 1}
			\sum_{i=1}^{\lfloor {k}/{2}\rfloor}\Big(\prod_{j=i}^{k}\rho_j\Big)  \frac{2M_1}{3}
			\le
			\sqrt{\lfloor {k}/{2}
				\rfloor + 1} \exp\Big(2 \sqrt{ \lfloor {k}/{2} \rfloor + 1} - 2\sqrt{k+1}\Big)  \frac{2M_1}{3} \le \frac{2 M_1}{\sqrt{k+2}} \ ,
		\end{equation}\end{small}where the second inequality above uses \cref{lm: simple ineq}.
		Similarly for the second term in the rightmost side of \cref{ineq: recurrent G_k 3} we have:
		\begin{small}\begin{equation}\label{ineq: G_k recurrent sumpro 2}
			\sum_{i=\lfloor k/2 \rfloor+1}^{k}\Big(\prod_{j=i}^{k}\rho_j\Big) \frac{2M_1}{\lfloor k/2\rfloor+2} \le   \sum_{i=1}^{k}\Big(\prod_{j=i}^{k}\rho_j\Big) \frac{2M_1}{\lfloor k/2\rfloor+2} \le \frac{\sqrt{k+1} \cdot  2M_1}{\lfloor k/2\rfloor+2} \le \frac{4M_1}{\sqrt{k+2}} \ , 
		\end{equation}\end{small}where the second inequality uses \cref{lm: rho}.
		Substituting \cref{ineq: G_k recurrent sumpro 1} and \cref{ineq: G_k recurrent sumpro 2} back into \cref{ineq: recurrent G_k 3} yields the first part of \cref{eq: lm: convex EZ}, namely:
		\begin{equation}\label{ineq: convex EG_k final}
			\mathbb{E} [G_k] \le {6 M_1}/{\sqrt{k+2}} \ .
		\end{equation}
		
		We employ a similar proof approach to establish the second part of \cref{eq: lm: convex EZ}. We note that $Z_k^i = 0$ when $i \in \calB_k$, which happens with probability $\beta_k/n$.  Because of the relationship between $\{\tau^k_i\}$ and $\{\tau^{k-1}_i\}$ in \cref{eq: tauk relationship}, we further note that with probability $1-\beta_k/n$ that $i \notin \calB_k$ in which case we have
		\begin{small}\begin{equation*}
			\begin{aligned} 
				Z_k^i &= (G_{k-1}^i + \gamma_{k-1}d_{i,k-1})^2 = Z_{k-1}^i + 2 G_{k-1}^i\gamma_{k-1}d_{i,k-1} + (\gamma_{k-1}d_{i,k-1})^2 \\ &\le Z_{k-1}^i + 2 M_\infty \gamma_{k-1}G_{k-1}^i + \gamma_{k-1}^2 M_\infty d_{i,k-1} \ . 
			\end{aligned}\end{equation*}\end{small}Therefore
		\begin{small}\begin{equation}\begin{aligned}\label{nclz}
				\mathbb{E}[ Z_k ] = &\sum_{i=1}^n \mathbb{E} [Z_k^i ] \le \rho_k \sum_{i=1}^n \mathbb{E} [Z_{k-1}^i + 2 M_\infty \gamma_{k-1}G_{k-1}^i + \gamma_{k-1}^2 M_\infty d_{i,k-1}] \\ 
				&\le \rho_k \left( \mathbb{E} [Z_{k-1}] + 2   M_\infty \gamma_{k-1}\mathbb{E}[G_{k-1}] +   \gamma_{k-1}^2 M_\infty M_1 \right) \\
				&\le \rho_k \left(\mathbb{E} [Z_{k-1}] + \frac{24M_\infty M_1}{(k+1)\sqrt{k+1}}   +  \frac{4M_\infty M_1}{(k+1)^2} \right) \le \rho_k \left(\mathbb{E} [Z_{k-1}] + \frac{28M_\infty M_1}{(k+1)\sqrt{k+1}}  \right)    \ ,
		\end{aligned}\end{equation} \end{small}where the second inequality uses $\sup_{j} \sum_{i=1}^n d_{ij} \le M_{1}$, and the third inequality uses \cref{ineq: convex EG_k final}. Recurrently applying \cref{nclz} yields:
		\begin{small}\begin{equation}\label{ineq: recurrent Z_k 4}
			\mathbb{E} [Z_k]  \le \rho_k\bigg( \mathbb{E} [Z_{k-1} ]+
			\frac{28M_1M_\infty}{(k+1)^{3/2}}\bigg) \le \cdots\le \sum_{i=1}^{k}\Big(\prod_{j=i}^{k}\rho_j\Big)\frac{28M_1M_\infty}{(i+1)^{3/2}}  \ .
		\end{equation}\end{small}Therefore
		\begin{small}\begin{equation}\label{ineq: recurrent Z_k 5}
			\begin{aligned}
				\mathbb{E} [Z_k]  & \le  \sum_{i=1}^{\lfloor k/2\rfloor}\Big(\prod_{j=i}^{k}\rho_j\Big)\frac{28M_1M_\infty}{(i+1)^{3/2}}
				+
				\sum_{i=\lfloor k/2\rfloor+1}^{k}\Big(\prod_{j=i}^{k}\rho_j\Big)\frac{28M_1M_\infty}{(i+1)^{3/2}}\\
				& \le \sum_{i=1}^{\lfloor k/2\rfloor}\Big(\prod_{j=i}^{k}\rho_j\Big)\frac{28M_1M_\infty}{3\sqrt{3}}
				+
				\sum_{i=\lfloor k/2\rfloor+1}^{k}\Big(\prod_{j=i}^{k}\rho_j\Big) \frac{28M_1M_\infty}{(\lfloor k/2\rfloor +2)^{3/2}} \ .
			\end{aligned}
		\end{equation}\end{small}By \cref{lm: rho} we have 
			\begin{small}\begin{equation}\label{ineq: Z_k recurrent sumpro 1}
			\begin{aligned}
				\sum_{i=1}^{\lfloor k/2\rfloor}\Big(\prod_{j=i}^{k}\rho_j\Big)\frac{28M_1M_\infty}{3\sqrt{3}}  &\le \sqrt{\lfloor k/2
					\rfloor + 1} \exp\big(2 \sqrt{ \lfloor k/2 \rfloor + 1} - 2\sqrt{k+1}\big) \cdot \frac{28M_1M_\infty}{(3\sqrt{3})} \\ 
				&\le \frac{{54M_1M_\infty}}{k+2} \ , 
		\end{aligned}\end{equation}\end{small}where the second inequality above uses \cref{lm: simple ineq}. 
		Similarly for the second term in the rightmost side of \cref{ineq: recurrent Z_k 5} we have:
			\begin{small}\begin{equation}
			\begin{aligned}\label{ineq: Z_k recurrent sumpro 2}
				\sum_{i=\lfloor k/2 \rfloor+1}^{k}\Big(\prod_{j=i}^{k}\rho_j\Big)\frac{28M_1M_\infty}{(\lfloor k/2\rfloor +2)^{3/2}} &\le   \sum_{i=1}^{k}\Big(\prod_{j=i}^{k}\rho_j\Big) \frac{28M_1M_\infty}{(\lfloor k/2\rfloor +2)^{3/2}} \\ 
				&\le \sqrt{k+1} \cdot \frac{28M_1M_\infty}{(\lfloor k/2\rfloor +2)^{3/2}} \le \frac{80M_1M_\infty}{k+2} \ ,
			\end{aligned}
		\end{equation}\end{small}where the second inequality uses \cref{lm: rho}.
		Substituting \cref{ineq: Z_k recurrent sumpro 1} and \cref{ineq: Z_k recurrent sumpro 2} back into \cref{ineq: recurrent Z_k 5} yields the second part of \cref{eq: lm: convex EZ}, namely:
		\begin{small}
			\begin{equation}\label{ineq: convex EZ_k final}
				\mathbb{E} [Z_k] \le  \frac{134M_1M_\infty}{k+2} \ .
			\end{equation}
		\end{small}\end{proof}
	
	Finally, we present the proof of \cref{thm: general convergence rate convex}.
	\begin{proof}[Proof of \cref{thm: general convergence rate convex}]
		From \cref{lm: general convergence rate convex} it holds that:
		\begin{small}
			\begin{equation}
				\eps_k  \le  \frac{2L D^2}{k+1}
				+  \frac{2}{k(k+1)}\sum_{t=1}^{k} t \left(
				\nabla F(x^{t-1})   - g^{t-1}
				\right)^\top (s^{t-1} - x^\star) \ ,
			\end{equation}
		\end{small}and from \cref{lm: general 3rd term upper bound} it follows that:
		\begin{small}
			\begin{equation*}
				\eps_k \le  \frac{2L D^2}{k+1}
				+  \frac{\hat{L}D^3}{nk(k+1)}\sum_{t=1}^{k}t\sum_{i=1}^{n}\Big( \sum_{j = \tau^{t-1}_i}^{t-2}  \gamma_j \Big)^2 \ .
			\end{equation*}
		\end{small}Taking expectations on both sides then yields
		\begin{small}
			\begin{equation}\label{ineq: convex thm 1 general}
				\mathbb{E} [\eps_k]  \le  \frac{2L D^2}{k+1}
				+  \frac{\hat{L}D^3}{nk(k+1)}\sum_{t=1}^{k}t~\mathbb{E}\Big[
				\sum_{i=1}^{n}\Big(  \sum_{j = \tau^{t-1}_i}^{t-2}  \gamma_j \Big)^2
				\Big] \ .
			\end{equation}
		\end{small}Now define $d_{ij} := 1$ and $M_1 := n$ and $M_\infty := 1$.  Then since $\sup_{j}   \sum_{i=1}^{n} d_{ij} = M_1$  and  $\sup_{i,j} d_{ij}$ is equal to $M_\infty$, we can apply \cref{lm: convex EZ} to \cref{ineq: convex thm 1 general}  and obtain
		\begin{small}
			\begin{equation}\label{ineq: convex thm 2 general}
				\mathbb{E} [\eps_k]   \le  \frac{2L D^2}{k+1}
				+  \frac{\hat{L}D^3}{nk(k+1)}\sum_{t=1}^{k}t~\frac{134 n}{t+1} \le \frac{2L D^2}{k+1}
				+  \frac{134 \hat{L}D^3}{k+1} \ , 
			\end{equation}
		\end{small}which proves \cref{eq of thm: general convergence rate convex}.
	\end{proof}

	Towards the proof of \cref{thm:  general convergence rate convex rule 3}, we first prove the following lemma.

	\begin{lemma}\label{lm: general convex EZ rule 3}
		In any iteration $k$ of TUFW outlined in \cref{alg: TUFW in general finite sum}  with $\gamma_k:=2/(k+2)$ and $\rthree$, it holds that 
		\begin{equation}\label{ineq lm general convex EZ rule 3}
			\sum_{t=1}^{k}t\sum_{i=1}^{n}\Big( \sum_{j = \tau^{t-1}_i}^{t-2}  \gamma_j\Big)^2  \le 144k n \ .
		\end{equation}
	\end{lemma}
	\begin{proof}[Proof of \cref{lm: general convex EZ rule 3}]
		Notice that in $\rthree$, $\tau^k_i = (\lfloor \sqrt{k} \rfloor)^2$, so $k - \tau^k_i = k - (\lfloor \sqrt{k} \rfloor)^2 \le \lfloor 2\sqrt{k} \rfloor$, where the latter follows from squaring the inequality $\lfloor \sqrt{k} \rfloor \ge \sqrt{k} -1$ and rearranging. Therefore,
		\begin{small}
			\begin{equation}\label{ineq: for general convergence rate convex 1}
				\begin{aligned}
					& \sum_{t=1}^{k}t\sum_{i=1}^{n}\Big( \sum_{j = \tau^{t-1}_i}^{t-2}  \gamma_j  \Big)^2 \le 
					\sum_{t=1}^{k}t\sum_{i=1}^{n}\Big(  (t-1 -\tau^{t-1}_i) \cdot \gamma_{\tau^{t-1}_i}  \Big)^2 \\
					\le & 
					\sum_{t=1}^{k}t\sum_{i=1}^{n}\Big( 4 (t-1) \cdot\frac{4 }{ (t+1 - \lfloor 2\sqrt{t-1} \rfloor)^2}  \Big) 
					\le  16 n \sum_{t = 1}^k \frac{t(t-1) }{ (t+1 - \lfloor 2\sqrt{t-1} \rfloor)^2} 
				\end{aligned}
			\end{equation}
		\end{small}where the first inequality is because $\gamma_j$ monotonically decreases as $j$ increases   and the second inequality is due to
		$\tau^{t-1}_i \ge t - 1 - \lfloor 2\sqrt{t-1} \rfloor$ and $k - \tau^k_i  \le 2\sqrt{k} $. 
		
		In the right-hand side of \cref{ineq: for general convergence rate convex 1}, we have $
		\frac{t  (t-1) }{ (t+1 - \lfloor 2\sqrt{t-1} \rfloor)^2} \le \frac{t^{2}}{t^2/9} = 9$,
		where the inequality follows from verifying $t+1 - \lfloor 2\sqrt{t-1} \rfloor \ge t/3 $ via the quadratic formula.  
		The inequality \cref{ineq lm general convex EZ rule 3} then follows by substituting this inequality into \cref{ineq: for general convergence rate convex 1}.
	\end{proof}

	\begin{proof}[Proof of \cref{thm: general convergence rate convex rule 3}]
		Similar to the proof of \cref{thm: general convergence rate convex}, according to \cref{lm: general convergence rate convex,lm: general 3rd term upper bound},
		\begin{small}\begin{equation*}
			\eps_k \le  \frac{2L D^2}{k+1}
			+  \frac{\hat{L}D^3}{nk(k+1)}\sum_{t=1}^{k}t\sum_{i=1}^{n}\Big( \sum_{j = \tau^{t-1}_i}^{t-2}  \gamma_j \Big)^2 \ .
		\end{equation*}\end{small}Next apply \cref{lm: general convex EZ rule 3} to the above inequality, which yields \cref{eq: thm: general convergence rate convex rule 3}.
	\end{proof}

\section{Convergence Guarantees for Non-convex Loss Functions}\label{sec: Convergence Guarantees non-convex}

In this section we study convergence rates and overall complexity of \cref{alg: TUFW in general finite sum} for tackling problem \cref{pro: ERM} in the case that the loss function $F$ is not necessarily convex.  

Recall that the Frank-Wolfe gap $\calG(x)$ at $x$ was defined in \cref{def: FW gap}, and when $F$ is convex $\calG(x^k)$ is an upper bound on the optimality gap at $x^k$, see \cite{jaggi2013revisiting} for example. The Frank-Wolfe gap does not necessarily bound the optimality gap in the non-convex setting; nevertheless $\calG(x)$ is always nonnegative and $\calG(x) =0$ if and only if $x$ is a stationary point of $\erm$, and for this reason $\calG(x)$ is often used as a measure of non-stationarity at $x$, see \cite{lacoste2016convergence,reddi2016stochastic,yurtsever2019conditional,negiar2020stochastic}. Hence for $\erm$ with non-convex objective function $F$, we say that a point $x \in \calC$ is an $\eps$-stationarity point of $\erm$ if $\calG(x)\le \eps$. 

Recall also that in order to run \cref{alg: TUFW in general finite sum}, we need to specify the step-sizes $\{\gamma_k\}_k$ and the \textbf{Rule} for constructing the sets $\calB_k$ of indices for updating the Taylor-points in \cref{fanette} of \cref{alg: TUFW in general finite sum}. Let $K$ denote the number of iterations that we intend to run.  In a somewhat similar spirit as the rule $\rone$ for convex losses, we present the following stochastic updating rule that is designed for $\calB_k$ to be comprised of $\beta_k:= n /\sqrt[4]{K}$ independently drawn samples from $[n]$ without replacement, for all $k\ge 1$. If $\beta_k$ is not integral, we will instead ensure $\mathbb{E} |\calB_k| = \beta_k$ by using a similar approach to $\rone$ which uses a Bernoulli random variable. With this in mind, we formally define the following \textbf{Rule}:

\begin{definition}\label{rule 2} $\rtwo$.
	For a fixed value of $K \ge 1$, and for all $k \ge 1$ define $\beta_k := \beta := n /\sqrt[4]{K}$ and $p_k:= p:= \beta_k - \lfloor \beta_k\rfloor$. Sample $\xi_k$ from the Bernoulli distribution $\mathrm{Ber}(p_k)$ and then uniformly sample $\lfloor \beta_k \rfloor + \xi_k$ samples from $[n]$ without replacement. Then define $\calB_k$ to be the set of these index samples.
\end{definition}

Note that by design of $\rtwo$ it also follows that $\mathbb{E}_{\xi_k} |\calB_k| = \beta_k$.  We call this rule ``$\rtwo$'' for ``Stochastic Batch-size Decreasing at the rate $\sqrt[4]{K}$,'' because the Taylor points are updated less often as the number of overall iterations grow. Next we present a bound on the total number of flops used in $K$ iterations of \cref{alg: TUFW in general finite sum} using $\rtwo$. 

\begin{proposition}\label{proposition: complextity of K iterations} Using $\rtwo$ and $K \ge 1$ iterations, the expected total number of flops used in \cref{alg: TUFW in general finite sum} is $O(K \cdot (\mathrm{fLMO} + p^2) +  K^{3/4} \cdot np^2)$.
\end{proposition}
\begin{proof}[Proof of \cref{proposition: complextity of K iterations}]
	For the initial iteration of \cref{alg: TUFW in general finite sum} the number of flops is $O(\mathrm{fLMO} + np^2)$, and the expected number of flops in each iteration thereafter is $O(\mathrm{fLMO} + p^2 + n p^2 /\sqrt[4]{K} )$. Summing over the $K$ iterations yields the result.
\end{proof}

The following theorem presents a computational guarantee for computing an $\eps$-stationarity point of $\erm$ using \cref{alg: TUFW in general finite sum} with $\rtwo$.  The proof of this theorem and other ensuing results in the section are presented in \cref{proofs-non-convex}.

\begin{theorem}
	\label{thm: general convergence rate non-convex}
	Suppose that \cref{assump: ERM} holds and $F$ is not necessarily convex, and let $x^\star$ be any optimal solution of \cref{pro: ERM}. Suppose \cref{alg: TUFW in general finite sum}  is applied to problem \cref{pro: ERM}, with $\rtwo$ and step-sizes defined by $\gamma_k  := \gamma := {1}/{\sqrt{K+1}}$ for all $k \ge 0$, where $K \ge 1$ is given. Then:
	\begin{equation}\label{eq: thm: general convergence rate non-convex}
		\mathbb{E}[\min_{k \in \{0, \ldots, K\}}\calG(x^k)] \le   \sum_{k=0}^{K}\frac{\mathbb{E}[\calG(x^k)]}{K+1} \le \frac{F(x^0) - F(x^*)}{\sqrt{K+1}} + \frac{3\hat{L}D^3 + LD^2}{2\sqrt{K+1}} \ .
	\end{equation}
\end{theorem}

The following corollary presents an alternative version of \cref{thm: general convergence rate non-convex} by randomly choosing an iteration index $\hat k \in [K]$, and shows that the joint dependence on $n$ and $\varepsilon$ is $O(n/\varepsilon^{3/2})$. 

\begin{corollary}\label{corollary: general n and eps non-convex} The number of flops required to obtain $\mathbb{E}_{\hat k \sim \calU([K])} \mathbb{E}[\calG(x^{\hat k})] \le \eps$ is 
	\begin{small}
		\begin{equation*}
			O\left( (\mathrm{fLMO} + p^2) \Bigg[
			\frac{\big(\eps_0  +  L D^2 + \hat{L}D^3\big)^2}{\eps^2 } 
			\Bigg] 
			+ 
			np^2 
			\Bigg[
			\frac{\big(\eps_0   +  L D^2 + \hat{L}D^3\big)^{3/2}}{\eps^{3/2} } 
			\Bigg]
			\right) \ .
		\end{equation*}
	\end{small}where $\eps_0 = F(x^0) - F(x^\star)$.
\end{corollary}
The proof of \cref{corollary: general n and eps non-convex} is an immediate consequence of \cref{proposition: complextity of K iterations} and \cref{thm: general convergence rate non-convex}. 

In the expectation $\mathbb{E}_{\hat k \sim \calU([K])} \mathbb{E}[\calG(x^{\hat k})]$ in \cref{corollary: general n and eps non-convex}, the randomness comes from both the random sampling at each iteration $k$ to construct the update batches $\calB_k$ in $\rtwo$ as well as the random sampling of the iteration index $\hat{k}$ from the uniform distribution on $[K]$. Analogous to the results for convex losses, \cref{corollary: general n and eps non-convex} implies that the joint dependence on $n$ and $\varepsilon$ is $O(n/\varepsilon^{3/2})$ for $\erm$ problems. In \cref{sec erml problems}, we will specify these computational guarantees to $\erml$ problems.

In a similar spirit as $\rthree$, we also have the following deterministic rule. 
\begin{definition}\label{rule 4} $\rfour $.  For a fixed $K \ge 1$, and for any $k\ge 1$, define 
		$$
			\calB_k = \left\{
			\begin{array}{cl}
				[n] & \text{ if  } \ k/\lfloor \sqrt[4]{K}\rfloor \in \mathbb{N}  \\
				\emptyset & \text{ if  } \ k/\lfloor \sqrt[4]{K}\rfloor \notin \mathbb{N} \ .
			\end{array}
			\right.
			$$
\end{definition}
In $\rfour$ we do not update any Taylor points unless $k$ is an integer times $\lfloor \sqrt[4]{K}\rfloor$, and for these values of $k$ we update all $n$ Taylor points. We refer the interested reader to the technical report \cite{punt} for computational guarantees for $\rfour$.

\cref{tbl: compare rules nonconvex} shows a comparison of the computational guarantees of the standard Frank-Wolfe method and TUFW with $\rtwo $.
\begin{table}[htbp]
	\begin{centering}
		\caption{Complexity bounds for different Frank-Wolfe methods to obtain an $\eps$-stationary solution of $\erm$ with non-convex losses. In the table $\eps_0:=F(x^0) - F(x^\star)$, $c_1:=L D^2$, and $c_2:=\hat{L}D^3$.}\label{tbl: compare rules nonconvex}
	\end{centering}
	\renewcommand\arraystretch{1.8} 
	\begin{adjustbox}{width=1\columnwidth,center}
		\begin{tabular}{ccc}
			Method                      & Optimality Metric                                                           & Overall Complexity                                                                                                                                              \\  \hline \hline
			$\rtwo$ (Corollary $\ref{corollary: general n and eps non-convex}$)           & $\mathbb{E}_{\hat{k}\sim \calU([K])}\mathbb{E}[\calG(x^{\hat k})] \le \eps$ & $\displaystyle  O\left( (\mathrm{fLMO}+p^2) \cdot \frac{\left( \eps_0 + c_1+c_2\right)^2}{\eps^2}  + np^2 \cdot \frac{\left( \eps_0+ c_1+c_2\right)^{3/2}}{\eps^{3/2}} \right)$ \\  \hline 
			standard Frank-Wolfe ($\cite{lacoste2016convergence}$)                    & $\mathbb{E}_{\hat{k}\sim\calU([K])}[\calG(x^{\hat k})] \le \eps$            & $\displaystyle  O\left( (\mathrm{fLMO}+np) \cdot \frac{\left(\eps_0+ c_1\right)^2}{\eps^2}\right)$                                                         
		\end{tabular}
	\end{adjustbox}
\end{table}
Note in particular that with $\rtwo$, the joint dependence on $n$ and $\eps$ is $O(n/\eps^{3/2})$, as compared to $O(n^{2/3}/\eps^2)$ in \cite{reddi2016stochastic}, $O(n^{1/2}/\eps^2)$ in \cite{yurtsever2019conditional}, and $O(1/\eps^3)$ in \cite{zhang2020one,hassani2020stochastic}.  In the regime where  $\eps$ is sufficiently small, the dominant term in the complexity bound is the left-most term which counts the operations of the LMO, which is $O(1/\eps^2)$ (independent of $n$), and is essentially the same as the lower bound complexity in \cite{lan2013complexity}. Because of this, we say that TUFW's overall complexity is nearly optimal. We also mention that the joint dependence on $n$ and $\eps$ for CASPIDERRG in \cite{shen2019complexities} is $O(n^{3/4}/\eps^{3/2})$; however CASPIDERG requires access to the exact full gradients periodically, and in our computational experience we found that CASPIDERG was not computationally viable relative to other stochastic Frank-Wolfe methods. 
It should also be mentioned that  the TUFW's joint dependence on $n$, $p$ and $\eps$ is $O(p^2/\eps^2 + np^2/\eps^{3/2})$, as compared to $O(n p/\eps^2)$ in \cite{lacoste2016convergence}. 
As long as $n$ is larger compared with $p$, i.e.  $n > p \cdot (\eps_0 + c_1 + c_2)^2 / (\eps_0 + c_1)^2$, and $\eps$ is sufficiently small,  TUFW outperforms the standard Frank-Wolfe method  in terms of the dependence on $n$ and $p$.

Finally, we point out that when the loss functions $f_i$ are all quadratic loss functions, the Hessian $\nabla^2 F(x)$ is constant for for all $x$, and we can use $\rfive$ which is the rule that does not update any of the Taylor points.

\subsection{Proofs of results}\label{proofs-non-convex}

In our proofs we will use two results,  \cref{prop: general basic ineq} and \cref{lm: general 3rd term upper bound}, since they do not assume convexity of $F$.  For the rest of this section let $\eps_0 := F(x^0) - F(x^\star)$.  We start with the following sequences of lemmas. 

\begin{lemma}\label{lm: general convergence rate nonconvex}
	Under Assumption \ref{assump: ERM} for problem \cref{pro: ERM}, if we use the step-size $\gamma_k :=\gamma := \frac{1}{\sqrt{K+1}}$ in Algorithm \ref{alg: TUFW in general finite sum}, and let $\bar{s}^k\in \arg\min_{s\in\calC}(\nabla F(x^k))^\top s$, then
	\begin{equation}\label{eq lm: general convergence rate nonconvex}
		\sum_{k=0}^{K}\frac{ \calG(x^k)}{K+1}  \le \frac{\varepsilon_0}{\sqrt{K+1}} + \frac{1}{K+1}\sum_{k=0}^{K} (\nabla F(x^k)-g^k)^\top ( s^{k}-\bar{s}^{k}) + \frac{LD^2}{2\sqrt{K+1}}.
	\end{equation}
\end{lemma}
\begin{proof}[Proof of \cref{lm: general convergence rate nonconvex}]
	For any $k \ge 0$ we have:
	\begin{small}
		\begin{equation}
			\begin{aligned}
				F(x^{k+1})
				\le~&  F(x^k) + \gamma_k (\nabla F(x^k))^\top (s^{k} - x^k) + {\gamma_k^2}L D^2/{2} \\
				=  ~&   F(x^k) + \gamma_k (g^k)^\top (s^{k} - x^k)+\gamma_k (\nabla F(x^k) - g^k)^\top (s^{k} - x^k) + {\gamma_k^2}L D^2/{2} \\
				\le  ~&   F(x^k) + \gamma_k (g^k)^\top (\bar{s}^{k} - x^k)+\gamma_k (\nabla F(x^k) - g^k)^\top (s^{k} - x^k) + {\gamma_k^2}L D^2/{2} \\
				=  ~&   F(x^k) - \gamma_k \calG(x^k) +\gamma_k (g^k-\nabla F(x^k))^\top (\bar{s}^{k} - s^{k}) + {\gamma_k^2}L D^2/{2} \ , \\
			\end{aligned}
		\end{equation}
	\end{small}where $\bar s^{k} \in\arg\min_{s\in\calC} (\nabla F(x^k)^\top s$, and the first inequality uses \cref{prop: general basic ineq}, and the second inequality uses the fact that $s^{k}\in\arg\min_{s\in\calC} (g^k)^\top s$. Substituting the value of $\gamma_k$ in the statement of the lemma yields after rearranging terms:
	\begin{small}\begin{equation}\label{ineq: basic ineq non-convex general}
		\frac{\calG(x^k)}{\sqrt{K+1}}  \le F(x^k) - F(x^{k+1}) + \frac{1}{\sqrt{K+1}} (\nabla F(x^k)-g^k)^\top ( s^{k}-\bar{s}^{k}) + \frac{LD^2}{2(K+1)} \ . 
	\end{equation}\end{small}The inequality \cref{eq lm: general convergence rate nonconvex} then follows by first summing the inequalities \cref{ineq: basic ineq non-convex general} for $k \in \{0,\dots,K\}$ and dividing by ${\sqrt{K+1}}$, and noting that $F(x^0) - F(x^{K+1}) \le \eps_0$.
\end{proof}

\begin{lemma}\label{lm:  non-convex EZ}
	Let the series $\{d_{ij}\}_{i\in[n],j\in\mathbb{N}}$ be given, and suppose that there exists $M_{1}$ and $M_{\infty}$ for which $\sup_{j} \sum_{i=1}^n d_{ij} \le M_{1}$ and  $\sup_{i,j} d_{ij}\le M_{\infty}$. 
	Let $\tau_i^k$ be defined as in \cref{def: tau}, where $\calB_t$ is constructed using $\rtwo$ in \cref{rule 2}. Let $\gamma_k := \gamma := {1}/{\sqrt{K+1}}$.  Then it holds that
	\begin{equation}\label{eq: lm:  non-convex EZ}
		\mathbb{E}\Big[  \sum_{i=1}^{n} \sum_{j = \tau^k_i}^{k-1}  \gamma_j d_{ij}  \Big]\le \frac{M_1}{(K+1)^{1/4}} \ \ \text{ and } \  \ \mathbb{E} \Big[\sum_{i=1}^{n}\Big( \sum_{j = \tau^k_i}^{k-1}  \gamma_j d_{ij} \Big)^2 \Big]\le \frac{3M_1M_\infty}{\sqrt{K+1}} \ . 
	\end{equation}
\end{lemma}
\begin{proof}[Proof of \cref{lm:  non-convex EZ}]
	Similar to the proof of \cref{lm: convex EZ}, let us first define $G_k^i := \left( \sum_{j = \tau^k_i}^{k-1}  \gamma_j d_{ij} \right)$ and $Z_k^i :=\left( \sum_{j = \tau^k_i}^{k-1}  \gamma_j d_{ij} \right)^2$ for $i \in [n]$, and also $G_k :=\sum_{i=1}^n G_k^i$ and $Z_k :=\sum_{i=1}^n Z_k^i$. First notice in $G_k^i$ and $Z_k^i$ that the randomness lies only in $\tau_i^k$, which in turn depends on the update sets $\calB_0, \calB_1,\dots,\calB_k$.  We note that $G_k^i = 0$ when $i \in \calB_k$, which happens with probability $\beta/n = 1/\sqrt[4]{K}$.  Because of the relationship between $\{\tau^k_i\}$ and $\{\tau^{k-1}_i\}$ in \cref{eq: tauk relationship}, we further note that with probability $1-\beta/n$ that $i \notin \calB_k$ in which case it follows that $G_k^i = G_{k-1}^i + \gamma_{k-1}d_{i,k-1} = G_{k-1}^i + \gamma d_{i,k-1}$.  Defining $\rho := 1 - \beta/n = 1-  1/\sqrt[4]{K}$, we have 
	\begin{small}
			\begin{equation}\label{ncl2}
			\mathbb{E}[ G_k ] = \sum_{i=1}^n \mathbb{E} [G_k^i ]= \rho \sum_{i=1}^n \mathbb{E} [G_{k-1}^i +
			\gamma d_{i,k-1}] \le \rho \left( \mathbb{E} [G_{k-1} ]+
			\gamma M_1 \right) \ ,
		\end{equation} 
	\end{small}where the inequality uses $\sup_{j} \sum_{i=1}^n d_{ij} \le M_{1}$.  Now \cref{ncl2} can be rearranged to yield:
	\begin{small}\begin{equation}\label{ineq: non-convex recurrent G_k 2}
		\mathbb{E} [G_k]
		-\frac{\rho \gamma M_1}{1-\rho}
		\le \rho\Big( \mathbb{E} [G_{k-1}]
		-\frac{\rho \gamma M_1}{1-\rho}
		\Big)  \le \cdots \le \rho^k \Big( \mathbb{E} [G_{0}]
		-\frac{\rho \gamma M_1}{1-\rho}
		\Big) \le 0 \ ,  
	\end{equation}\end{small}where the last inequality uses $G_0 = 0$ (since all Taylor points are set to $x^0$ in \cref{bettina} of \cref{alg: TUFW in general finite sum}).  We therefore have:
	\begin{small}\begin{equation}\label{ineq: non-convex recurrent G_k 3}
			  \mathbb{E} [G_k] \le \frac{\rho \gamma M_1}{1-\rho} = \frac{(1-K^{-1/4})(K+1)^{-1/2} M_1  }{K^{-1/4}} \le \frac{M_1}{(K+1)^{1/4}} \ , 
	\end{equation}\end{small}which proves the first inequality of \cref{eq: lm:  non-convex EZ}.
	
	We proceed similarly to establish the second inequality of \cref{eq: lm:  non-convex EZ}.  We note that $Z_k^i = 0$ when $i \in \calB_k$, which happens with probability $\beta/n = 1/\sqrt[4]{K}$.  Because of the relationship between $\{\tau^k_i\}$ and $\{\tau^{k-1}_i\}$ in \cref{eq: tauk relationship}, we further note that with probability $\rho = 1-\beta/n$ that $i \notin \calB_k$ in which case it follows that $Z_k^i = (G_{k-1}^i + \gamma d_{i,k-1})^2$.  We thus have
	\begin{small}
		\begin{equation}\label{ncl3}\begin{aligned}
				& \mathbb{E}[ Z_k ] =  \sum_{i=1}^n \mathbb{E} [Z_k^i ] \\
				= \ &\rho \sum_{i=1}^n \mathbb{E} [(G_{k-1}^i + \gamma d_{i,k-1})^2] 
				=  \rho  \sum_{i=1}^n \left( \mathbb{E} [Z_{k-1}^i ]+ 2\gamma \mathbb{E} [G_{k-1}^i] d_{i,k-1} + \gamma^2 (d_{i,k-1})^2]  \right) \\ 
				\le \ &  \rho \left( \mathbb{E} [Z_{k-1} ]+ 2\gamma \mathbb{E} [G_{k-1}] M_\infty + \gamma^2 (M_\infty M_1)  \right) \le
				\rho \left( \mathbb{E} [Z_{k-1} ]+ 2\gamma \frac{M_1 M_\infty}{(K+1)^{1/4}}   + \gamma^2 M_\infty M_1  \right)  \\   
				= \ &  \rho \left( \mathbb{E} [Z_{k-1} ]+ 2 \frac{M_1 M_\infty}{(K+1)^{3/4}}   + \frac{M_\infty M_1}{K+1}  \right) \le \rho \left( \mathbb{E} [Z_{k-1} ]+  \frac{3M_1 M_\infty}{(K+1)^{3/4}}   \right) \ . 
		\end{aligned}  \end{equation}
	\end{small}Now \cref{ncl3} can be rearranged to yield:
	\begin{small}	\begin{equation}\label{ineq: non-convex recurrent Z_k 2}
		\begin{aligned}
				\mathbb{E} [Z_k]
				-\frac{3M_1M_\infty}{(K+1)^{3/4}} \cdot \frac{\rho}{1-\rho} &\le \rho\Big( \mathbb{E} [Z_{k-1}]
				-\frac{3M_1M_\infty}{(K+1)^{3/4}} \cdot \frac{\rho}{1-\rho} 
				\Big) \\ 
				&\le \cdots \le \rho^k \Big( \mathbb{E} [Z_{0}]
				-\frac{3M_1M_\infty}{(K+1)^{3/4}} \cdot \frac{\rho}{1-\rho}
				\Big) \le 0 \ ,  \end{aligned}
	\end{equation}\end{small}where the last inequality uses $Z_0 = 0$ (since all Taylor points are set to $x^0$ in \cref{bettina} of \cref{alg: TUFW in general finite sum}), from which it follows that 
	\begin{small}\begin{equation}\label{ineq: non-convex recurrent Z_k 5}
		 \mathbb{E} [Z_k] \le \frac{3M_1M_\infty}{(K+1)^{3/4}} \cdot \frac{\rho}{1-\rho} \le \frac{3M_1M_\infty}{\sqrt{K+1}} \ , 
	\end{equation}\end{small}which proves the second inequality of \cref{eq: lm:  non-convex EZ}. \end{proof}

\begin{proof}[Proof of \cref{thm: general convergence rate non-convex}]
	The first inequality of \cref{eq: thm: general convergence rate non-convex} follows from the chain
	\begin{small}$$  \mathbb{E}[\min_{k \in \{0, \ldots, K\}}\calG(x^k)] \le \min_{k \in \{0, \ldots, K\}}  \mathbb{E}[\calG(x^k)] \le \sum_{k=0}^{K}\frac{\mathbb{E}[\calG(x^k)]}{K+1}  \ . $$\end{small}Applying \cref{lm: general 3rd term upper bound} to \cref{eq lm: general convergence rate nonconvex} of \cref{lm: general convergence rate nonconvex} we obtain:
	\begin{small}
		\begin{equation}
			\sum_{k=0}^{K}\frac{ \calG(x^k)}{K+1}  \le \frac{2\eps_0 + LD^2}{2\sqrt{K+1}} + \frac{\hat{L}D^3}{2n(K+1)}\sum_{k=0}^{K}  \sum_{i=1}^{n}\Big( \sum_{j = \tau^k_i}^{k-1}  \gamma_j\Big)^2.
	\end{equation}\end{small}Taking expectations on both sides, it follows that:
	\begin{small}
		\begin{equation}\label{ineq: general non-convex thm 1}
			\sum_{k=0}^{K}\frac{ \mathbb{E}[\calG(x^k)]}{K+1}  \le \frac{2\eps_0 + LD^2}{2\sqrt{K+1}} + \frac{\hat{L}D^3}{2n(K+1)} \sum_{k=0}^{K}   \mathbb{E}\bigg[  \sum_{i=1}^{n}\Big( \sum_{j = \tau^k_i}^{k-1}  \gamma_j\Big)^2  \bigg]. 
	\end{equation}\end{small}Now define $d_{ij} := 1$ and $M_1 := n$ and $M_\infty := 1$.  Then since $\sup_{j}   \sum_{i=1}^{n} d_{ij}= M_1$,  and  $\sup_{i,j} d_{ij}  = M_\infty$, we can apply the second inequality of \cref{lm:  non-convex EZ} to \cref{ineq: general non-convex thm 1} and we thus obtain the second inequality of \cref{eq: thm: general convergence rate non-convex}.  \end{proof}

\section{Specializing to $\erml$ Problems}\label{sec erml problems}

In this section, we specialize our results to the setting of $\erml$ problems \cref{pro: ERM with linear prediction} -- namely problems whose loss functions are characterized by linear prediction. We will analyze $\rone$ (for convex loss functions) and $\rtwo$ (for nonconvex loss functions) and show the corresponding computational guarantees for \cref{alg: TUFW in general finite sum}.

Because the feasibility set $\calC$ is bounded, the following measure of the range of the linear operator $W^\top$
\begin{equation}\label{def: D_q}
	 D_q := \max_{x,y\in \calC} \left\|W^\top (x-y)\right\|_q 
\end{equation} 
is finite  for all $q \in  [1,\infty]$. We note that \cref{def: D_q} is a common metric used in the literature of problem \cref{pro: ERM with linear prediction}, see \cite{lu2021generalized,negiar2020stochastic}.  
We have the following computational guarantees for \cref{alg: TUFW in general finite sum} when specialized to $\erml$ in the convex and non-convex cases, respectively.

\begin{theorem}\label{thm: convergence rate convex}
	Suppose that $F$ is convex and \cref{assump: ERM with linear prediction} holds, and \cref{alg: TUFW in general finite sum} with $\rone$ is applied to the problem \cref{pro: ERM with linear prediction} with step-sizes defined by $\gamma_k:={2}/({k+2})$ for all $k \ge 0$. Then for all $k\ge 1$ we have:
	\begin{equation}\label{eq: thm: convergence rate convex}
		\mathbb{E}[F(x^{k}) - F(x^\star) ] \le  \frac{2L D_2^2 + 134 \hat{L}D_1D_{\infty}^2}{n(k+1)} \ . 
	\end{equation}
\end{theorem}
\begin{theorem}
	\label{thm: convergence rate non-convex}
	Suppose that \cref{assump: ERM with linear prediction} holds and $F$ is not necessarily convex, and \cref{alg: TUFW in general finite sum} with $\rtwo$ is applied to the problem \cref{pro: ERM with linear prediction} with step-sizes defined by $\gamma_k  := \gamma := {1}/{\sqrt{K+1}}$ for all $k \ge 0$, where $K \ge 1$ is given. Then:
	\begin{equation}\label{eq: thm: convergence rate non-convex}
		\mathbb{E}[\min_{k \in \{0, \ldots, K\}}\calG(x^k)] \le   \sum_{k=0}^{K}\frac{\mathbb{E}[\calG(x^k)]}{K+1} \le \frac{F(x^0) - F(x^*)}{\sqrt{K+1}} + \frac{3\hat{L}D_1D_\infty^2 + LD_2^2}{2n\sqrt{K+1}} \ . 
	\end{equation}
\end{theorem}

Using \cref{karljunior} and \cref{proposition: complextity of K iterations}, we can then bound the total number of flops as follows.

\begin{corollary} \label{lori}
	Under the hypotheses of \cref{thm: convergence rate convex}, it holds that $\mathbb{E}[F(x^{k}) - F(x^\star) ] \le \varepsilon$ after at most $ \frac{2L D_2^2 + 134 \hat{L}D_1D_{\infty}^2}{n \varepsilon} $
	iterations, and the total number of flops required is at most
	\begin{equation*} O\left( (\mathrm{fLMO} + p^2) \left[\frac{C}{n \varepsilon} \right] + np^2
		\left[\frac{\sqrt{C}}{\sqrt{n} \sqrt{\varepsilon}} \right] 
		\right) \ ,  \end{equation*}
	where $C:=L D_2^2 + \hat{L}D_1D_{\infty}^2$.
\end{corollary}

\begin{corollary}\label{corollary: n and eps non-convex} Under the hypotheses of \cref{thm: convergence rate non-convex}, the bound on the number of flops required to obtain $\mathbb{E}_{\hat k \sim \calU([K])} \mathbb{E}[\calG(x^{\hat k})] \le \eps$ is 
		\begin{equation*}
			O\left( (\mathrm{fLMO} + p^2) \left[
			\frac{C^2}{n^2\eps^2 } 
			\right]+ 
			np^2 
			\left[
			\frac{C^{3/2}}{n^{3/2}\eps^{3/2} } 
			\right]
			\right) \ ,
		\end{equation*}
where $C := n(F(x^0) - F(x^*))   +  L D_2^2 + \hat{L} D_1 D_\infty^2$.
\end{corollary}

When the boundness assumption on $w_i $  in \cref{lm transfer assumption} holds, then $D_2 \le \sqrt{n} D_\infty \le  \sqrt{n} M  D$ and $D_1 \le n D_\infty \le  n M  D$, and  \cref{lori} directly yields \cref{cor general convex erm} for $\erml$, which implies that these results on $\erml$ problems are potentially stronger than directly applying the results for $\erm$ problems to $\erml$ problems, due to the different measurement of the diameter of $\calC$.

\subsection{Proofs of results}\label{subsection proof erml}

We will prove \cref{thm: convergence rate convex} and \cref{thm: convergence rate non-convex} as a consequence of the following sequence of lemmas.

\begin{lemma}\label{lm: basic ineq 1}
	Under \cref{assump: ERM with linear prediction} for problem $\erml$, 
	\begin{equation}\label{ineq: general convergence rate with I(s)}
		F(x^{k+1}) \le  F(x^k)+\gamma_{k}\left\langle\nabla F(x^k),  s^{k}- x^k\right\rangle+\gamma_{k}^{2} {L D_{2}^2}/{2n}
	\end{equation}
	holds for iteration $k$ of \cref{alg: TUFW in general finite sum} for all $k \ge 0$.
\end{lemma}
\begin{proof}[Proof of \cref{lm: basic ineq 1}]
	
	For iteration $k$ it follows from \cref{assump: eq: ERM structured l'} that
	\begin{small}
		\begin{equation*}
			l_i(w_i^\top x^{k+1})\le  l_i(w_i^\top x^{k}) +  l_i'(w_i^\top x^{k})(w_i^\top x^{k+1} - w_i^\top x^{k}) + {L}(w_i^\top x^{k+1} - w_i^\top x^{k})^2/2 \ .
		\end{equation*}
	\end{small}Then since $F(x^{k+1})  =\frac{1}{n}\sum_{i=1}^{n} l_i(w_i^\top x^{k+1})$, we have:
	\begin{small}\begin{equation*}
		\begin{aligned}
			F(x^{k+1})\le ~ & F(x^k)+\gamma_{k}\langle \nabla F(x^k),  s^{k}- x^k\rangle+\gamma_{k}^{2} L\big\|W^\top s^{k}-W^\top x^k\big\|_{2}^{2}/2n \\
			\le ~ & F(x^k)+\gamma_{k}\langle \nabla F(x^k),  s^{k}- x^k\rangle+\gamma_{k}^{2} L D_2^{2}/2n \ , 
		\end{aligned}
	\end{equation*}\end{small}where the last inequality above uses \cref{def: D_q}. \end{proof}

\begin{lemma}\label{lm: 3rd term upper bound}
	Under \cref{assump: ERM with linear prediction} for the problem $\erml$, for any $u,v \in\calC$,
	\begin{equation}\label{eq: lm: 3rd term upper bound}
		\left(
		\nabla F(x^k)   - g^k
		\right)^\top (u - v) \le \frac{\hat{L}D_{\infty}}{2n}\sum_{i=1}^{n}\Big( \sum_{j = \tau^k_i}^{k-1}  \gamma_j\left|w_{i}^\top (s^{j} -  x^{j})\right|\Big)^2 \ .
	\end{equation}
	holds for iteration $k$ of \cref{alg: TUFW in general finite sum} for all $k \ge 1$.
\end{lemma}

\begin{proof}[Proof of \cref{lm: 3rd term upper bound}]
	It follows from the definition of $\tau^k_i$ in \cref{def: tau} that 
	\begin{small}
		\begin{equation*}\label{ineq: etimate gradient 1}
			\begin{array}{l}
				\left(
				\nabla F(x^k)   - g^k
				\right)^\top (u - v) \\ 
				= \frac{1}{n}\sum_{i=1}^{n}\big(l_{i}'\big(w_{i}^\top x^k \big) -  l_{i}'\big(w_{i}^\top x^{\tau^{k}_i}\big) -   l_{i}''\big(w_{i}^\top x^{\tau^{k}_i}\big)\big(w_i^\top x^k  - w_i^\top x^{\tau^{k}_i} \big)    \big)\cdot w_{i}^\top (u - v) \\
				\le  \frac{\hat{L}}{2n}\sum_{i=1}^{n} |w_{i}^\top x^k - w_{i}^\top x^{\tau^k_i} |^2 \cdot |w_{i}^\top(u - v)| \\
				=  \frac{\hat{L}}{2n}\sum_{i=1}^{n} \big( \sum_{j = \tau^k_i}^{k-1} w_{i}^\top (x^{j+1} -  x^{j}) \big)^2 \cdot |w_{i}^\top(u - v)| \\
				\le  \frac{\hat{L}}{2n}\sum_{i=1}^{n} \big( \sum_{j = \tau^k_i}^{k-1} \left|w_{i}^\top (x^{j+1} -  x^{j})\right| \big)^2 \cdot |w_{i}^\top(u - v)| \\
				=  \frac{\hat{L}}{2n}\sum_{i=1}^{n} \big( \sum_{j = \tau^k_i}^{k-1} \left| \gamma_jw_{i}^\top (s^{j} -  x^{j})\right| \big)^2 \cdot |w_{i}^\top(u - v)|
			\end{array}
		\end{equation*}
	\end{small}where the first inequality uses \cref{assump: eq: ERM structured l''}. 
	We then use $|w_{i}^\top(u - v)|  \le D_\infty$ to obtain \eqref{eq: lm: 3rd term upper bound}.
\end{proof}


\begin{lemma}\label{lm: convergence rate convex}
Under \cref{assump: ERM with linear prediction} for problem $\erml$, suppose that 
$F$ is convex, and let the step-sizes be given by $\gamma_k = {2}/(k+2)$ for all $k \ge 0$.  Then for all $k \ge 1$ we have 
\begin{equation}\label{eq: lm: convergence rate convex}
	\eps_k  \le  \frac{2L D_2^2}{n (k+1)}
	+  \frac{2}{k(k+1)}\sum_{t=1}^{k} t \left(
	\nabla F(x^{t-1})   - g^{t-1}
	\right)^\top (s^{t-1} - x^\star) \ . 
\end{equation}
\end{lemma}
\begin{proof}[Proof of \cref{lm: convergence rate convex}] The proof is nearly identical to that of \cref{lm: general convergence rate convex}, and follows by replacing \cref{prop: general basic ineq} with  \cref{lm: basic ineq 1} and replacing $D^2$ with $D_2^2/n$. 
\end{proof}

\begin{proof}[Proof of \cref{thm: convergence rate convex}]
Applying \cref{lm: 3rd term upper bound} to \cref{eq: lm: convergence rate convex}, it follows that: 
\begin{small}
	\begin{equation*}
		\eps_k \le  \frac{2L D_2^2}{n (k+1)}
		+  \frac{\hat{L}D_{\infty}}{nk(k+1)}\sum_{t=1}^{k}t\sum_{i=1}^{n}\Big( \sum_{j = \tau^{t-1}_i}^{t-2}  \gamma_j\left|w_{i}^\top (s^{j} -  x^{j})\right|\Big)^2 \ .
	\end{equation*}
\end{small}Taking expectations on both sides then yields
\begin{small}
	\begin{equation}\label{ineq: convex thm 1}
		\mathbb{E} [\eps_k]  \le  \frac{2L D_2^2}{n (k+1)}
		+  \frac{\hat{L}D_{\infty}}{nk(k+1)}\sum_{t=1}^{k}t~\mathbb{E}\Big[
		\sum_{i=1}^{n}\Big( \sum_{j = \tau^{t-1}_i}^{t-2}  \gamma_j\left|w_{i}^\top (s^{j} -  x^{j})\right|\Big)^2
		\Big] \ .
	\end{equation}
\end{small}Now define $d_{ij} := \left|w_{i}^\top (s^{j} -  x^{j})\right|$ and $M_1 := D_1$ and $M_\infty := D_\infty$.  Then since $\sup_{j}   \sum_{i=1}^{n} d_{ij} = \sup_{j}   \sum_{i=1}^{n}  \left|w_{i}^\top (s^{j} -  x^{j})\right|\le D_1 = M_1$  and  $\sup_{i,j} d_{ij}$ is equal to $\sup_{i,j} \left|w_{i}^\top (s^{j} -  x^{j})\right|$ which is bounded by $ D_\infty = M_\infty$, we can apply \cref{lm: convex EZ} to \cref{ineq: convex thm 1}  and obtain
\begin{small}\begin{equation}\label{ineq: convex thm 2}
		\mathbb{E} [\eps_k]   \le  \frac{2L D_2^2}{n (k+1)}
		+  \frac{\hat{L}D_{\infty}}{nk(k+1)}\sum_{t=1}^{k}t~\frac{134 D_1D_\infty}{t+1} \le \frac{2L D_2^2}{n (k+1)}
		+  \frac{134 \hat{L}D_1D_{\infty}^2}{n(k+1)} \ , 
	\end{equation}
\end{small}which demonstrates \cref{eq: thm: convergence rate convex}.
\end{proof}

Before proving \cref{thm: convergence rate non-convex}, we present the following lemma. 
\begin{lemma}\label{lm: convergence rate non-convex}
	Under \cref{assump: ERM with linear prediction} for problem  \cref{pro: ERM with linear prediction}, if we use the step-size $\gamma_k  := \gamma := {1}/{\sqrt{K+1}}$ in \cref{alg: TUFW in general finite sum}, and let $\bar{s}^{k} \in\arg\min_{s\in\calC} (\nabla F(x^k)^\top s$, then
	\begin{equation}\label{eq: lm: convergence rate non-convex}
		\sum_{k=0}^{K} \frac{\calG(x^k)}{K+1} \le \frac{\eps_0}{\sqrt{K+1}} + \frac{1}{K+1}\sum_{k=0}^{K} (\nabla F(x^k)-g^k)^\top ( s^{k}-\bar{s}^{k}) + \frac{LD_2^2}{2n\sqrt{K+1}} \ .
	\end{equation}
\end{lemma}

\begin{proof}[Proof of \cref{lm: convergence rate non-convex}] The proof is nearly identical to that of \cref{lm: general convergence rate nonconvex}, and follows by replacing \cref{prop: general basic ineq} with  \cref{lm: basic ineq 1} and replacing $D^2$ with $D_2^2/n$.
\end{proof}


\begin{proof}[Proof of \cref{thm: convergence rate non-convex}]
	The first inequality of \cref{eq: thm: convergence rate non-convex} directly follows from \cref{thm: general convergence rate non-convex}.
	Applying \cref{lm: 3rd term upper bound} to \cref{eq: lm: convergence rate non-convex} of  \cref{lm: convergence rate non-convex} we obtain:
	\begin{small}
		\begin{equation}
			\sum_{k=0}^{K}\frac{ \calG(x^k)}{K+1}  \le \frac{2n\eps_0 + LD_2^2}{2n\sqrt{K+1}} + \frac{\hat{L}D_\infty}{2n(K+1)}\sum_{k=0}^{K}  \sum_{i=1}^{n}\Big( \sum_{j = \tau^k_i}^{k-1}  \gamma_j\left|w_{i}^\top (s^{j} -  x^{j})\right|\Big)^2.
	\end{equation}\end{small}Taking expectations on both sides, it follows that:
	\begin{small}
		\begin{equation}\label{ineq: non-convex thm 1}
			\sum_{k=0}^{K}\frac{ \mathbb{E}[\calG(x^k)]}{K+1}  \le \frac{2n\eps_0 + LD_2^2}{2n\sqrt{K+1}} + \frac{\hat{L}D_\infty}{2n(K+1)} \sum_{k=0}^{K}   \mathbb{E}\bigg[  \sum_{i=1}^{n}\Big( \sum_{j = \tau^k_i}^{k-1}  \gamma_j\left|w_{i}^\top (s^{j} -  x^{j})\right|\Big)^2  \bigg]. 
	\end{equation}\end{small}Now define $d_{ij} := \left|w_{i}^\top (s^{j} -  x^{j})\right|$ and $M_1 := D_1$ and $M_\infty := D_\infty$.  Then since $\sup_{j}   \sum_{i=1}^{n} d_{ij}$ is equal to $\sup_{j}   \sum_{i=1}^{n}  \left|w_{i}^\top (s^{j} -  x^{j})\right|$ which is at most $D_1 = M_1$,  and  $\sup_{i,j} d_{ij}$ is equal to $\sup_{i,j} \left|w_{i}^\top (s^{j} -  x^{j})\right|$ which is at most $D_\infty = M_\infty$, we can apply the second inequality of \cref{lm:  non-convex EZ} to \cref{ineq: non-convex thm 1} and we obtain the second inequality of \cref{eq: thm: convergence rate non-convex}.  \end{proof}

\section{Adaptive Step-size Rules}\label{sec: Extensions}

Note that the results in the previous two sections are based on non-adaptive step-size rules, in part because the methods only work with approximate gradients and so objective function improvement is more difficult to prove.  However, because of the extra smoothness of the second derivatives of the loss functions $l_i$, it turns out that one can build an accurate-enough local quadratic model of each $l_i$ which will make an adaptive step-size work well. Consider the following adaptive step-size for TUFW:
\begin{equation}\label{eq: adaptive step size}
	\tilde{\gamma}_k := 
	\left\{
	\begin{array}{ll}
		\min\left\{\gamma_k,    \frac{(g^k)^\top (x^k - s^k)}{   (s^k -  x^{k})^\top {H_k}(s^k -  x^{k})}  \right\} & \text{ if } (s^k -  x^{k})^\top {H_k}(s^k -  x^{k}) > 0  \ , \\
		\gamma_k & \text{ if } (s^k -  x^{k})^\top {H_k}(s^k -  x^{k}) \le 0 \ ,  
	\end{array}\right.
\end{equation}
where $H_k$ is defined in \cref{eq: intermedia variables} and $\gamma_k$ is the standard step-size. Our computational results in \cref{exper} show the huge advantage of \cref{eq: adaptive step size} over the comparable non-adaptive step-size, and computational guarantees for \cref{eq: adaptive step size} are presented in the technical report \cite{punt}.

\section{Computational Experiments}\label{exper}

In this section we present the results of computational experiments where we compare Algorithm TUFW (with the four different batch construction Rules developed earlier) and five other methods, namely the standard Frank-Wolfe method \cite{frank1956algorithm,jaggi2013revisiting} plus related methods in the recent literature, as follows:
\begin{itemize}
	\item (FW) -- the standard Frank-Wolfe method, see \cite{frank1956algorithm,jaggi2013revisiting}
	\item (FW-ada) -- the standard Frank-Wolfe method with adaptive step-size \cite{dem1967minimization}, which is defined by $\gamma_k := \min\{1,\calG(x^k)/(L \|x^k - s^k\|^2) \}$, where $L$ is the Lipschitz constant
	\item (SPIDER-FW) -- the Frank-Wolfe method with stochastic path-integrated differential estimator technique developed in \cite{yurtsever2019conditional}
	\item (CSFW) -- the constant batch-size stochastic Frank-Wolfe method developed in \cite{negiar2020stochastic}, and 
	\item (CASPIDERG) -- curvature-aided stochastic path-integrated differential
	estimator gradient developed in \cite{shen2019complexities}.
\end{itemize}

We tested these methods on both convex and non-convex instances of $\erm$.   We first conducted computational tests on problems with convex losses, using instances of the $\ell_1$-constrained logistic regression problem (LR): 
\begin{equation}\label{paul}
	   \mbox{LR:}  \ \ \  \min_{x}~  F(x):=\tfrac{1}{n} \sum_{i=1}^{n} \ln(1+\exp( - y_i w_i^\top x)) ~~\text{ s.t. } \|x\|_1\le \lambda \ , 
\end{equation}
where the $\ell_1$-ball constraint is designed to induce sparsity and/or regularization.  We chose instances of LR from LIBSVM \cite{chang2011libsvm} for which the number of observations $n \gg 0$ but not so large as to render the problem intractable. This yielded $12$ instances from LIBSVM, namely {\tt a1a}, {\tt a2a}, {\tt a8a}, {\tt a9a}, {\tt w1a}, {\tt w2a}, {\tt w7a}, {\tt w8a}, {\tt svmguide3}, {\tt phishing}, {\tt ijcnn1}, and {\tt covtype}, whose dimensions $n$ and $p$ are displayed in the third and forth columns of Table \ref{tbl: convex results CV tau}.  The value of $\lambda$ was specified by $5$-fold cross-validation of each instance.  However, we used the same cross-validated $\lambda$ for {\tt a1a}, {\tt a2a}, {\tt a8a}, and {\tt a9a}, and similarly for {\tt w1a}, {\tt w2a}, {\tt w7a}, for {\tt w8a}.  (Separately, we also tested the methods on larger feasible regions by enlarging each instance's feasbile region using $\lambda' :=100 \lambda$.) All methods were implemented and run in Python on the MIT Engaging Cluster. To ensure fairness, our implementations did not use sparse linear algebra for any methods, nor did we implement the techniques discussed in the last paragraph of section \ref{tufw}.

We measured the performance of methods on problems with convex losses using the Frank-Wolfe gap $\calG(x^k)$ defined in \cref{def: FW gap}, since 
$\calG(x^k)$ is a (conservative) upper bound on the optimality gap, namely $F(x^k) - F^* \le \calG(x^k)$ (see \cite{jaggi2013revisiting}), and the optimal value $F^*$ is not known.  We computed $\calG(x^k)$ separately and offline after-the-fact.

Note that $\rone$ and $\rtwo$ for TUFW require uniform sampling of $\bar\beta_k := |\calB_k|$ indices from $[n]$ without replacement.  These indices are then used to perform matrix and vector operations on the corresponding indexed columns of $W$.  Because matrix and vector operations involving random indices are inherently inefficient in practice, we instead implemented $\rone$ and $\rtwo$ by uniformly sampling a starting index $\hat i \sim \calU([n])$ and then assigning the other $|\bar\beta_k | - 1$ indices successively to construct $\calB_k := \{ \hat i, \hat i +1, \ldots, \hat i +  \bar\beta_k -1 \}$ (modulo $n$). This modification does not affect any of the  theoretical convergence guarantees because \cref{ncl} and \cref{nclz} still hold and thus \cref{lm: convex EZ} is still true, and similarly  \cref{ncl2} and \cref{ncl3} still hold and thus \cref{lm:  non-convex EZ} is still true.

\cref{fig:compare TUFW} shows the performance of TUFW with the rules $\rone$ and $\rthree$, and the standard Frank-Wolfe method, on the dataset {\tt a9a} (which we found to be representative of typical performance).  The methods were run both with and without adaptive step-sizes.  Notice that the TUFW methods significantly outperform the Frank-Wolfe method.  $\rthree$ using adaptive step-sizes significantly outperforms the other methods in terms of overall runtime and has similar LMO dependence as $\rone$. The second-best performance is from TUFW with $\rone$ with adaptive step-sizes.  As these results are typical, we will focus on $\rone$ and $\rthree$ with adaptive step-sizes in further comparisons with other Frank-Wolfe methods.
\begin{figure}[htbp]
	\centering
	\includegraphics[width=0.7\linewidth]{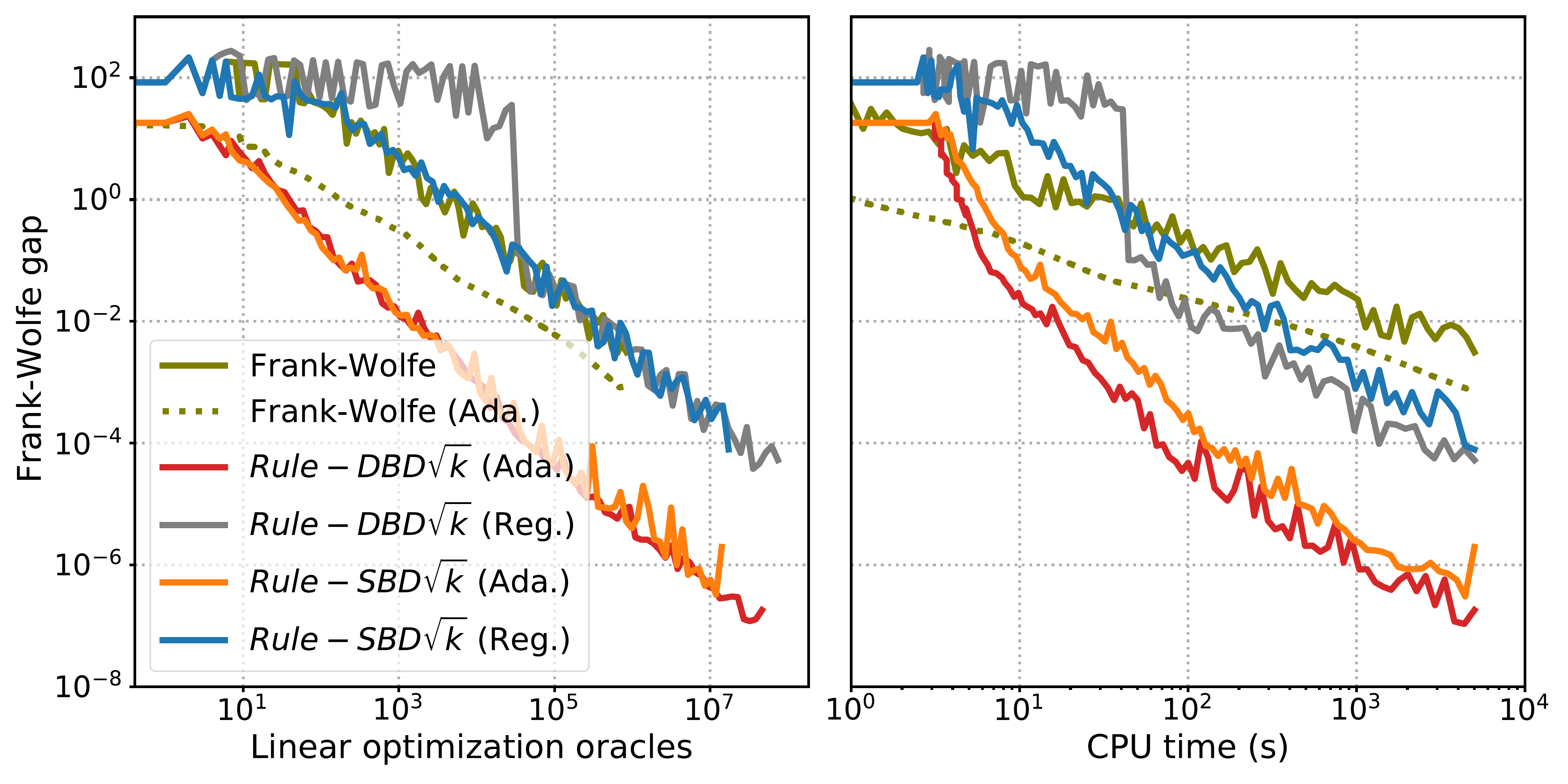}
	\caption{Performance of TUFW with $\rone$ and $\rthree$ and the standard Frank-Wolfe method, with and without adaptive step-sizes, on the logistic regression problem \cref{paul}, on the dataset {\tt a9a}.}
	\label{fig:compare TUFW}
\end{figure}

In \cref{tbl: convex results CV tau} we compare the CPU runtimes (in seconds) of the TUFW methods using $\rone$ and $\rthree$, with FW, FW-ada, SPIDER-FW, and CSFW. (We did not include CASPIDERG in these tests as it is not designed for convex objective functions.) The batch-size of CSFW was set to $n/100$, which is exactly as in the implementation in \cite{negiar2020stochastic}. For each method and each dataset we ran $10$ trials, and we report the average of these trials.  We only report average runtimes that did not exceed 5000 seconds, otherwise the entry is blank. In each row of the table the bold number highlights the best runtime among the six methods.  The last column of the table reports the speed-up of TUFW (the best of $\rone$ and $\rthree$) over the best of the other four methods.  Speed-up values larger than $1$ are shown in bold.  We observe from \cref{tbl: convex results CV tau} that the TUFW method -- using either $\rone$ or $\rthree$, significantly outperforms the other methods on almost all problem instances and almost all target optimality tolerances.  Furthermore, as the optimality tolerance $\eps$ decreases and/or as the number of observations $n$ increase, the advantage of TUFW methods over the other methods becomes even more evident. Additional computational experiments with TUFW are presented in \cite{punt}.

\begin{table}[htbp]
	\centering
	\caption{Comparison of average CPU runtimes (in seconds) required to achieve $\calG(x^k) \le \eps$ for methods on the logistic regression problem \cref{paul} with $\lambda$ chosen by cross-validation.  (A blank indicates the method used more than 5000 seconds.) }\label{tbl: convex results CV tau}
	\begin{adjustbox}{width=1\columnwidth,center}
		\begin{tabular}{|llll|llllll|l|}
			\hline
			$\eps$ & dataset & $n$ & $p$ & $\rone$ & $\rthree$ & FW & FW-ada & SPIDER-FW & CSFW & Speed-up \\ 
			\hline
			1e-1&  a1a & 1605 & 123&   0.73&   \textbf{0.37  }&   2.02&   1.18&   348.58&   5.03 & \textbf{3.20}\\
			1e-3&  a1a & 1605 & 123&   14.40&   \textbf{5.14  }&   138.87&   380.10&     &   160.02 & \textbf{27.02}\\
			1e-5&  a1a & 1605 & 123&   3029.61&   \textbf{1034.66  }&   2956.50&     &     &      & \textbf{2.86}\\
			\hline
			1e-1&  a2a & 2265 & 123&   1.02&   \textbf{0.52  }&   3.03&   1.55&   358.72&   5.47 & \textbf{2.97}\\
			1e-3&  a2a & 2265 & 123&   21.24&   \textbf{8.16  }&   197.80&   472.72&     &   167.66 & \textbf{20.54}\\
			1e-5&  a2a & 2265 & 123&   2039.75&   \textbf{631.23  }&     &     &     &      &  \\
			\hline
			1e-1&  a8a & 22696 & 123&   7.75&   \textbf{4.66  }&   78.18&   12.83&   443.85&   15.22 & \textbf{2.75}\\
			1e-3&  a8a & 22696 & 123&   51.29&   \textbf{23.71  }&     &   3178.02&     &   1250.12 & \textbf{52.72}\\
			1e-5&  a8a & 22696 & 123&   420.64&   \textbf{174.87  }&     &     &     &      &  \\
			\hline
			1e-1&  a9a & 32561 & 123&   11.26&   \textbf{6.80  }&   94.50&   21.00&   440.73&   19.82 & \textbf{2.91}\\
			1e-3&  a9a & 32561 & 123&   65.25&   \textbf{31.22  }&     &   3844.13&     &   1660.13 & \textbf{53.17}\\
			1e-5&  a9a & 32561 & 123&   490.36&   \textbf{197.15  }&     &     &     &      &  \\
			\hline
			1e-1&  w1a & 2477 & 300&   10.45&   \textbf{5.71  }&   17.51&   331.22&   3566.83&   44.39 & \textbf{3.07}\\
			1e-3&  w1a & 2477 & 300&   51.51&   \textbf{21.93  }&   327.38&     &     &   1219.03 & \textbf{14.93}\\
			1e-5&  w1a & 2477 & 300&   1743.08&   \textbf{541.71  }&     &     &     &      &  \\
			\hline
			1e-1&  w2a & 3470 & 300&   13.11&   \textbf{6.80  }&   37.34&   574.42&     &   54.06 & \textbf{5.50}\\
			1e-3&  w2a & 3470 & 300&   128.14&   \textbf{50.77  }&   369.82&     &     &   1207.01 & \textbf{7.28}\\
			1e-5&  w2a & 3470 & 300&     &   \textbf{1422.93  }&     &     &     &      &  \\
			\hline
			1e-1&  w7a & 24692 & 300&   88.86&   \textbf{55.43  }&   541.10&     &     &   129.19 & \textbf{2.33}\\
			1e-3&  w7a & 24692 & 300&   320.17&   \textbf{178.87  }&     &     &     &   4413.77 & \textbf{24.68}\\
			1e-5&  w7a & 24692 & 300&   3106.98&   \textbf{1516.83  }&     &     &     &      &  \\
			\hline
			1e-1&  w8a & 49749 & 300&   174.89&   \textbf{114.66  }&   1198.35&     &     &   212.80 & \textbf{1.86}\\
			1e-3&  w8a & 49749 & 300&   553.94&   \textbf{355.23  }&     &     &     &      &  \\
			1e-5&  w8a & 49749 & 300&     &   \textbf{2707.92  }&     &     &     &      &  \\
			\hline
			1e-1&  svmguide3 & 1243 & 22&   0.09&   \textbf{0.02  }&   1.35&   0.31&   21.46&   2.06 & \textbf{13.43}\\
			1e-3&  svmguide3 & 1243 & 22&   8.01&   \textbf{2.64  }&   53.64&   175.49&     &   48.85 & \textbf{18.53}\\
			1e-5&  svmguide3 & 1243 & 22&   1132.50&   \textbf{410.63  }&   629.39&     &     &   1484.23 & \textbf{1.53}\\
			1e-7&  svmguide3 & 1243 & 22&     &     &     &     &     &      &  \\
			\hline
			1e-1&  phishing & 11055 & 68&   0.47&   0.38&   0.01&   0.71&   0.01&   \textbf{0.01  } & 0.02\\
			1e-3&  phishing & 11055 & 68&   2.81&   1.24&   0.67&   201.65&   \textbf{0.32  }&   0.53 & 0.26\\
			1e-5&  phishing & 11055 & 68&   45.36&   \textbf{16.10  }&   47.01&     &   1391.81&   43.70 & \textbf{2.71}\\
			1e-7&  phishing & 11055 & 68&   2478.92&   \textbf{812.38  }&   3370.12&     &     &      & \textbf{4.15}\\
			\hline
			1e-1&  ijcnn1 & 49990 & 22&   0.85&   0.24&   1.21&   9.02&   \textbf{0.22  }&   0.92 & 0.91\\
			1e-3&  ijcnn1 & 49990 & 22&   4.15&   \textbf{0.86  }&   54.72&   565.14&   193.34&   71.65 & \textbf{63.40}\\
			1e-5&  ijcnn1 & 49990 & 22&   7.67&   \textbf{1.66  }&     &   2322.76&     &      & \textbf{1402.31}\\
			1e-7&  ijcnn1 & 49990 & 22&   10.20&   \textbf{2.32  }&     &   3910.71&     &      & \textbf{1688.65}\\
			\hline
			1e-1&  covtype & 581012 & 54&   54.62&   29.76&   143.55&   123.79&   \textbf{5.47  }&   18.91 & 0.18\\
			1e-3&  covtype & 581012 & 54&   552.47&   \textbf{231.70  }&   2964.49&     &   843.60&   690.78 & \textbf{2.98}\\
			1e-5&  covtype & 581012 & 54&     &   \textbf{3777.45  }&     &     &     &      &  \\
			\hline
		\end{tabular}
	\end{adjustbox}
\end{table}

We next tested the methods on non-convex instances of $\erml$ using instances of the $\ell_1$-constrained binary classification problem with least-squares thresholding using the sigmoid function, namely: 
\begin{equation}\label{leslie}
	 \mbox{BCP:} \ \ \ \ \ \    \min_{x} ~ \tfrac{1}{n} \sum_{i=1}^{n} \left(y_i - \frac{1}{1+\exp(-w_i^\top x)}\right)^2 ~~\text{ s.t. } \|x\|_1\le \lambda \ , 
\end{equation}
using the same $12$ datasets from LIBSVM as described earlier.  Recall that the Frank-Wolfe gap $\calG(x)$ is a measure of non-stationarity, see \cite{lacoste2016convergence,reddi2016stochastic,yurtsever2019conditional,negiar2020stochastic}.  In the non-convex setting the design of most Frank-Wolfe methods involves setting the number of iterations $K$ in advance and using a fixed step-size (usually $1/\sqrt{K+1}$), and the convergence rate for the Frank-Wolfe gap is typically measured using the expectation of $\calG(x^{\hat k})$ where $\hat k$ is uniformly sampled over $\{0, \ldots, K \}$.  In our experiments we therefore compute the average Frank-Wolfe gap by computing $\sum_{i=0}^K \calG(x^i)/(K+1)$ and then averaging this value over $5$ independent trials.  In order to save on computation when $K \ge 100n$, we only compute and record the Frank-Wolfe gaps every $200$ iterations.  Just as in the convex case, we computed the Frank-Wolfe gaps separately and offline after-the-fact. 

For each  $K \in   \{10, 20, 10 \times 2^2, \ldots,10\times 2^{20}\}$, we ran $5$ independent trials of the methods on the $12$ datasets. \cref{fig:compare TUFW nonconvex} shows the performance of TUFW with $\rtwo$ and $\rfour$, and the standard Frank-Wolfe method, on the dataset {\tt a9a}, both with and without adaptive step-sizes.  Each dot in the right-most subfigure corresponds to the average Frank-Wolfe gap, and the average runtime (over $5$ trials) using a particular value of $K$ as described above. Similar to the convex case, the TUFW methods significantly outperform the Frank-Wolfe method.  $\rfour$ using adaptive step-sizes exhibits the best performance over the other methods in terms of  overall runtime and has nearly identical LMO dependence as $\rtwo$. The second-best performance is from TUFW with $\rtwo$ with adaptive step-sizes.  We focus on $\rtwo$ and $\rfour$ with adaptive step-sizes in further comparisons with other Frank-Wolfe methods. 

\begin{figure}[htbp]
	\centering
	\includegraphics[width=0.7\linewidth]{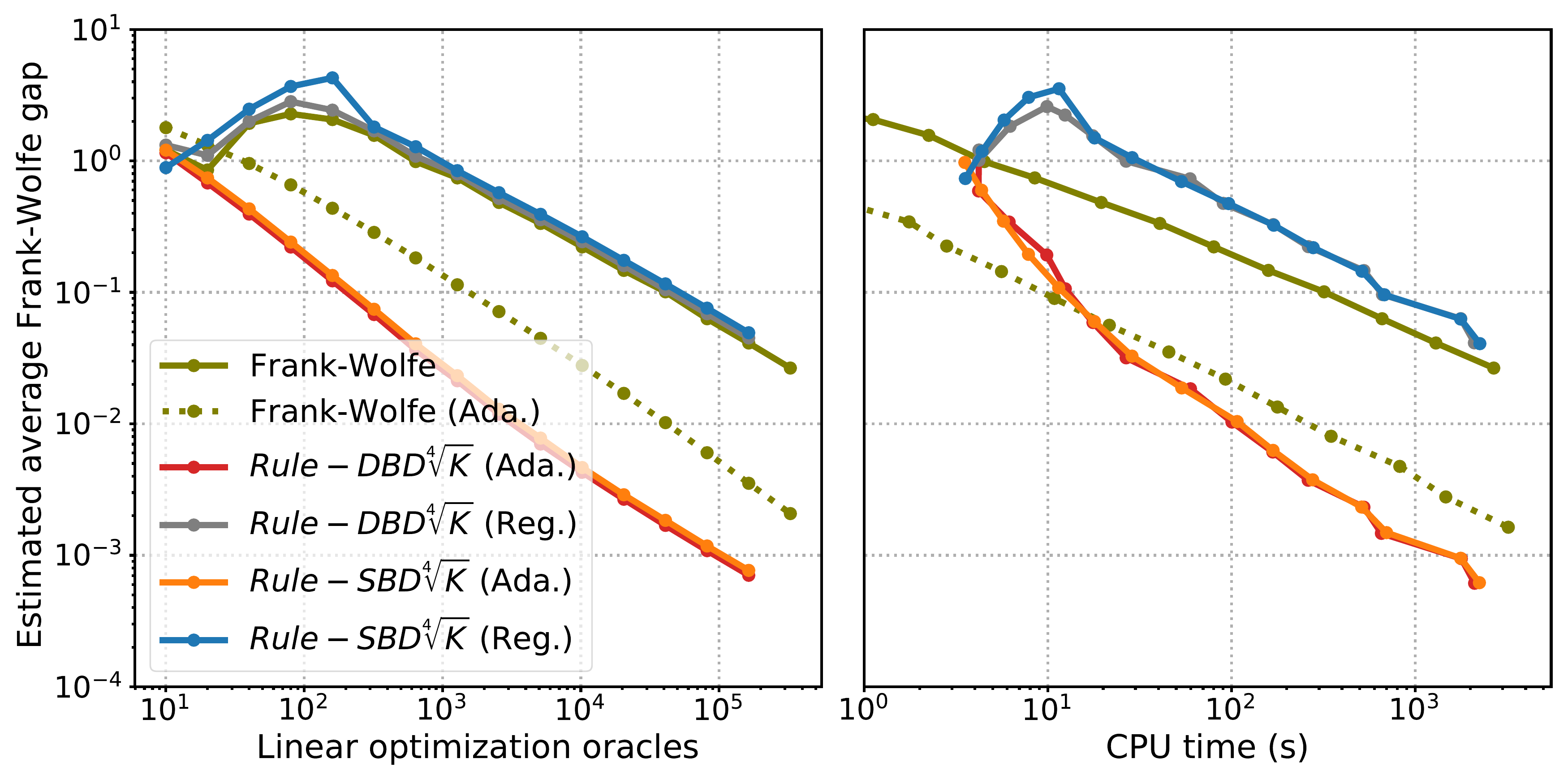}
	\caption{Performance of TUFW with $\rtwo$ and $\rfour$ and the standard Frank-Wolfe method, with and without adaptive step-sizes, on the binary classification problem \cref{leslie}, on the dataset {\tt a9a}.}
	\label{fig:compare TUFW nonconvex}
\end{figure}

In \cref{tbl: nonconvex results CV tau} we compare the CPU runtimes (in seconds) of the TUFW methods using $\rtwo$ and $\rfour$, with FW, FW-ada, SPIDER-FW, and CASPIDERG. CASPIDERG calculates full gradients every $n^{1/4}K^{1/4}$ iterations and its batch-size is set to $\max\{n^{3/4} K^{-1/4},1\}$.  (We did not include CSFW in these tests as it is not designed for non-convex objective functions.) All methods were run for each of the $21$ different values of $K$ and were terminated if they exceeded $5000$ seconds. Further details on these experiments are presented in \cite{punt}. In each row of the table the bold number highlights the best runtime among the six methods.  The last column of the table reports the speed-up of TUFW (the best of $\rtwo$ and $\rfour$) over the best of the other four methods.  Speed-up values larger than $1$ are shown in bold.  Notice from \cref{tbl: nonconvex results CV tau} that the TUFW method -- using either $\rtwo$ or $\rfour$ -- significantly outperforms the other methods on on the datasets {\tt a1a}, {\tt a2a}, {\tt a8a}, and {\tt a9a}, for sufficiently small values of $\eps$.  For the datasets {\tt w1a}, {\tt w2a}, {\tt w7a}, and {\tt w8a}, the TUFW methods are generally dominated by one of the other metods, with FW being the best on these datasets when $\eps$ is smaller.  On the datasets {\tt svmguide3}, {\tt phishing}, {\tt ijcnn1}, and {\tt covtype} the TUFW method outperforms the other methods, all the moreso as the required $\eps$ gets smaller.   

\begin{table}[htbp]
	\centering
	\caption{Comparison of average CPU runtimes (in seconds) required to achieve $\tfrac{1}{K+1}\sum_{k=0}^K\calG(x^k) \le \eps$ for methods on the non-convex binary classification problem \cref{leslie}.  (A blank indicates the method used more than 5000 seconds or $K \ge  10 \times  2^{20}$.)}\label{tbl: nonconvex results CV tau}
	\begin{adjustbox}{width=1\columnwidth,center}
		\begin{tabular}{|llll|llllll|l|}
			\hline
			$\calG(x^k)$ & dataset & $n$ & $p$ & $\rtwo$  & $\rfour$  & FW & FW-ada & SPIDER-FW & CASPIDERG & Speed-up
			\\
			\hline
			1e-2                &  a1a & 1605 & 119&   6.05 &   \textbf{4.44}&   10.46&   9.00&   15.77&                                   & \textbf{2.03}  \\
			1e-3                &  a1a & 1605 & 119&   90.34&   \textbf{65.11}&   818.40&   237.93&                                  &                                   & \textbf{3.65}  \\
			1e-4                &  a1a & 1605 & 119&   2225.61&   \textbf{1453.87}&                                  &   4953.53&                                  &                                   & \textbf{3.41}  \\
			\hline
			1e-2                &  a2a & 2265 & 119&   6.47&   \textbf{4.94}&   12.62&   10.55&   13.54&                                   & \textbf{2.14}  \\
			1e-3                &  a2a & 2265 & 119&   93.17&   \textbf{70.69}&   781.05&   303.67&                                  &                                   & \textbf{4.30}  \\
			1e-4                &  a2a & 2265 & 119&   2168.72&   \textbf{1389.71}&                                  &               &                                  &                                   &    \\
			\hline
			1e-2                &  a8a & 22696 & 123&   46.68&   41.92&   243.81&   140.10&   \textbf{35.70}&                                   & 0.85  \\
			1e-3                &  a8a & 22696 & 123&   668.35&   \textbf{603.88}&                                  &   3317.18&                                  &                                   & \textbf{5.49}  \\
			1e-4                &  a8a & 22696 & 123&                                  &                                  &                                  &                                  &                                  &                                   &    \\
			\hline
			1e-2                &  a9a & 32561 & 123&   83.24&   79.91&   358.08&   240.37&   \textbf{46.05}&                                   & 0.58  \\
			1e-3                &  a9a & 32561 & 123&   1165.35&   \textbf{1106.16}&                                  &           &                                  &                                   &    \\
			1e-4                &  a9a & 32561 & 123&                                  &                                  &                                  &                                  &                                  &                                   &    \\
			\hline
			5e-2                &  w1a                                & 2477 & 300&   8.99&   8.82&   \textbf{0.44}&   12.66&   0.96&   101.83 & 0.05  \\
			1e-2                &  w1a                                & 2477 & 300&   74.94&   102.18&   \textbf{4.69}&   157.90&   19.96&                                   & 0.06  \\
			2e-3                &  w1a                                & 2477 & 300&   1278.96&   4063.80&   \textbf{109.23}&   1768.32&                                  &                                   & 0.09  \\
			\hline
			5e-2                &  w2a                                & 3470 & 300&   12.64&   11.90&   \textbf{0.50}&   17.41&   1.01&   120.51 & 0.04  \\
			1e-2                &  w2a                                & 3470 & 300&   93.47&   122.01&   \textbf{7.44}&   226.70&   21.08&                                   & 0.08  \\
			2e-3                &  w2a                                & 3470 & 300&   900.09&   2545.77&   \textbf{152.15}&   2415.59&                                  &                                   & 0.17  \\
			\hline
			5e-2                &  w7a                                & 24692 & 300&   95.86&   90.12&   7.79&   271.26&   \textbf{2.16}&   595.38 & 0.02  \\
			1e-2                &  w7a                                & 24692 & 300&   544.57&   561.90&   80.96&   3596.39&   \textbf{29.75}&                                   & 0.05  \\
			2e-3                &  w7a                                & 24692 & 300&   4078.15&               &   \textbf{1631.47}&                                  &   2990.71&                                   & 0.40  \\
			\hline
			5e-2                &  w8a                                & 49749 & 300&   222.71&   216.21&   16.41&   534.20&   \textbf{3.25}&   1122.68& 0.02  \\
			1e-2                &  w8a                                & 49749 & 300&   1292.84&   1303.19&   176.70&                                  &   \textbf{41.53}&                                   & 0.03  \\
			2e-3                &  w8a                                & 49749 & 300&                                  &                                  &   \textbf{3268.55}&                                  &                                  &                                   &    \\
			\hline
			1e-1                &  svmguide3 & 1243 & 22&   0.47&   \textbf{0.15}&   2.84&   1.96&                                  &                                   & \textbf{13.39}  \\
			1e-2                &  svmguide3 & 1243 & 22&   7.95&   \textbf{2.41}&                                  &   74.73&                                  &                                   & \textbf{30.98}  \\
			1e-3                &  svmguide3 & 1243 & 22&   122.45&   \textbf{33.68}&                                  &   2946.23&                                  &                                   & \textbf{87.49}  \\
			1e-4                &  svmguide3 & 1243 & 22&   2418.83&   \textbf{804.17}&                                  &                                  &                                  &                                   &    \\
			\hline
			1e-1                &  phishing & 11055 & 68&   1.59&   1.17&   2.36&   102.48&   \textbf{1.17}&                                   & 0.99  \\
			1e-2                &  phishing & 11055 & 68&   17.53&   \textbf{12.18}&   158.49&   2506.03&                                  &                                   & \textbf{13.01}  \\
			1e-3                &  phishing & 11055 & 68&   216.38&   \textbf{154.79}&                                  &                                  &                                  &                                   &    \\
			1e-4                &  phishing & 11055 & 68&   5049.56&   \textbf{3558.37}&                                  &                                  &                                  &                                   &    \\
			\hline
			1e-1                &  ijcnn1 & 49990 & 22&   2.32&   \textbf{0.96}&   10.88&   103.03&   1.58&   368.77 & \textbf{1.64}  \\
			1e-2                &  ijcnn1 & 49990 & 22&   24.26&   \textbf{10.00}&   728.57&   2315.87&                                  &                                   & \textbf{72.85}  \\
			1e-3                &  ijcnn1 & 49990 & 22&   298.45&   \textbf{179.40}&                                  &                                  &                                  &                                   &    \\
			1e-4                &  ijcnn1 & 49990 & 22&                                  &   \textbf{2750.09}&                                  &                                  &                                  &                                   &    \\
			\hline
			5e-2                &  covtype                                & 581012 & 54&   156.41&   \textbf{110.28}&   2152.50&   2894.39&                                  &                                   & \textbf{19.52}  \\
			1e-2                &  covtype                                & 581012 & 54&   785.43&   \textbf{572.97}&                                  &                                  &                                  &                                   &    \\
			2e-3                &  covtype                                & 581012 & 54&   4292.90&   \textbf{3142.24}&                                  &                                  &                                  &                                   &    \\
			\hline
		\end{tabular}
	\end{adjustbox}
\end{table}

\bibliographystyle{siamplain}
\bibliography{references}

\newpage

\end{document}